\documentclass[accepted]{uai2021} % for initial submission
\usepackage[american]{babel}
\usepackage{natbib} % has a nice set of citation styles and commands
\usepackage{mathtools} % amsmath with fixes and additions
\usepackage{booktabs} % commands to create good-looking tables
\usepackage{tikz} % nice language for creating drawings and diagrams
\usepackage{algorithm}
\usepackage[noend]{algorithmic}
\usepackage{amsmath}
\usepackage{amssymb}
\usepackage{amsthm}
\usepackage{appendix}
\usepackage{bbm}
\usepackage{bm}
\usepackage{color}
\usepackage{dirtytalk}
\usepackage{dsfont}
\usepackage{enumerate}
\usepackage{graphicx}
\usepackage{listings}
\usepackage{subcaption}
\usepackage{times}
\usepackage{xspace}
\usepackage{comment}

\usepackage[capitalize,noabbrev]{cleveref}

\usepackage[backgroundcolor=white]{todonotes}
\Crefname{corollary}{Corollary}{Corollaries}
\Crefname{proposition}{Proposition}{Propositions}
\Crefname{theorem}{Theorem}{Theorems}
\Crefname{definition}{Definition}{Definitions}
\Crefname{assumption}{Assumption}{Assumptions}
\Crefname{example}{Example}{Examples}
\Crefname{remark}{Remark}{Remarks}
\Crefname{setting}{Setting}{Settings}
\Crefname{lemma}{Lemma}{Lemmas}

\usepackage{thmtools}
\declaretheorem[name=Theorem,refname={Theorem,Theorems},Refname={Theorem,Theorems}]{theorem}
\declaretheorem[name=Lemma,refname={Lemma,Lemmas},Refname={Lemma,Lemmas},sibling=theorem]{lemma}

\newcommand{\core}{\ensuremath{\tt CORe}\xspace}
\newcommand{\lincore}{\ensuremath{\tt LinCORe}\xspace}

\newcommand{\giro}{\ensuremath{\tt Giro}\xspace}
\newcommand{\linphe}{\ensuremath{\tt LinPHE}\xspace}
\newcommand{\lints}{\ensuremath{\tt LinTS}\xspace}
\newcommand{\linucb}{\ensuremath{\tt LinUCB}\xspace}
\newcommand{\phe}{\ensuremath{\tt PHE}\xspace}
\newcommand{\ts}{\ensuremath{\tt TS}\xspace}

\newcommand{\ucb}{\ensuremath{\tt UCB1}\xspace}
\mathchardef\mhyphen="2D
\newcommand{\ucbv}{\ensuremath{\tt UCB\mhyphen V}\xspace}

\newcommand{\npts}{\ensuremath{\tt NP\mhyphen TS}\xspace}
\newcommand{\ssmc}{\ensuremath{\tt SSMC}\xspace}

\DeclareMathOperator*{\argmax}{arg\,max\,}

\newcommand{\cF}{\mathcal{F}}

\newcommand{\cN}{\mathcal{N}}

\newcommand{\eps}{\varepsilon}

\newcommand{\E}[1]{\mathbb{E} \left[#1\right]}
\newcommand{\condE}[2]{\mathbb{E} \left[#1 \,\middle|\, #2\right]}
\newcommand{\Et}[1]{\mathbb{E}_t \left[#1\right]}
\newcommand{\prob}[1]{\mathbb{P} \left(#1\right)}
\newcommand{\condprob}[2]{\mathbb{P} \left(#1 \,\middle|\, #2\right)}
\newcommand{\probt}[1]{\mathbb{P}_t \left(#1\right)}
\newcommand{\var}[1]{\mathrm{var} \left[#1\right]}
\newcommand{\condvar}[2]{\mathrm{var} \left[#1 \,\middle|\, #2\right]}

\newcommand{\Mod}[1]{\ \mathrm{mod}\ #1}

\newcommand{\abs}[1]{\left|#1\right|}

\newcommand{\I}[1]{\mathds{1} \! \left\{#1\right\}}

\newcommand{\normw}[2]{\|#1\|_{#2}}

\newcommand{\set}[1]{\left\{#1\right\}}

\newcommand{\T}{^\top}
\newcommand{\realset}{\mathbb{R}}

\title{\core: Capitalizing On Rewards in Bandit Exploration}

\author{
\textbf{\large{Nan Wang}} \\ \normalsize{University of Virginia} \\ \normalsize{nw6a@virginia.edu}
\and
\textbf{\large{Branislav Kveton}} \\ \normalsize{Google Research} \\ \normalsize{bkveton@google.com}
\and
\textbf{\large{Maryam Karimzadehgan}} \\ \normalsize{Google Research} \\ \normalsize{maryamk@google.com}
}

\begin{document}
\maketitle
\begin{abstract}
We propose a bandit algorithm that explores purely by randomizing its past observations. In particular, the sufficient optimism in the mean reward estimates is achieved by exploiting the variance in the past observed rewards. We name the algorithm \textbf{C}apitalizing \textbf{O}n \textbf{Re}wards (\core). The algorithm is general and can be easily applied to different bandit settings. The main benefit of \core is that its exploration is fully data-dependent. It does not rely on any external noise and adapts to different problems without parameter tuning. We derive a $\tilde O(d\sqrt{n\log K})$ gap-free bound on the $n$-round regret of \core in a stochastic linear bandit, where $d$ is the number of features and $K$ is the number of arms. Extensive empirical evaluation on multiple synthetic and real-world problems demonstrates the effectiveness of \core. 
\end{abstract}

\section{Introduction}
\label{sec:introduction}
A \emph{multi-armed bandit} \citep{Lai1985asymptotically, lattimore2020bandit} is an online sequential decision-making problem, where the learning agent chooses actions represented by arms in an $n$-round game. After an arm is pulled, the agent receives a \emph{stochastic} reward generated from an unknown reward distribution associated with the arm. The goal of the agent is to maximize the expected $n$-round reward. As the agent needs to learn the mean rewards of the arms by pulling them, it faces the so-called exploitation-exploration dilemma: \emph{exploit}, and pull the arm with the highest estimated mean reward thus far; or \emph{explore}, and learn more about the arms. 

A \emph{stochastic linear bandit (or linear bandit)} \citep{paat2008linearly,yasin2011improved} is a generalization of a multi-armed bandit where each arm is associated with a feature vector. The mean reward of an arm is the dot product of its feature vector and an unknown parameter vector, which needs to be learned by the agent. A multi-armed bandit can be considered as a special case of linear bandits, where the feature vector of each arm is a one-hot vector indicating the index of the arm, and the parameter vector is a vector of corresponding mean rewards.  

Arguably, the most popular and well-studied exploration strategies for solving bandit problems are \emph{Thompson sampling} (\ts) \citep{thompson1933on,agrawal2013furtherTS} and \emph{Optimism in the Face of Uncertainty (OFU)} \citep{auer2002finite}. \ts maintains a posterior distribution over each arm's mean reward and samples from it to explore. This is efficient and has strong empirical performance when the posterior has a closed form \citep{chapelle2011empirical}. However, if the posterior does not have a closed form, as in many non-linear problems \citep{maccullagh1984generalized,filippi2010parametric}, it needs to be approximated, which is typically computationally expensive and limits the applicability of \ts \citep{gopalan2013thompson, abeille2016linear, riquelme2018deep}. On the other hand, OFU-based algorithms \citep{auer2002finite}, depend on the construction of high-probability confidence sets. They are theoretically near-optimal in multi-armed bandit and linear bandits. However, as the confidence sets are often constructed for worst-case scenarios, they are empirically less competitive. In addition, in some problems, such as generalized linear bandits or neural network bandits \citep{zhou2020neural}, it is only possible to approximate the confidence sets \citep{filippi2010parametric, zhang2016online, provably2017li}. These approximations affect the statistical efficiency of the algorithms and often perform poorly.

%They achieve near-optimal performance in multi-armed bandits \citep{garivier2011kl-ucb,agrawal2013furtherTS} and linear bandits \citep{yasin2011improved,agrawal2012thompson}. 
%However, as their success depends on specially designed mean estimators for different problems, they are usually difficult to generalize to complex problems, e.g., non-linear bandits. Specifically, \ts samples the estimated mean reward from the posterior mean distribution. This is efficient in the case of a closed form of the posterior distribution. However, if the posterior distribution does not have any closed form, it needs to be approximated, which is usually computationally expensive \citep{filippi2010parametric,abeille2016linear}. Similarly, OFU depends on the construction of high-probability confidence sets, which can be obtained in multi-armed and linear bandits. Nevertheless, for most complex problems, it is not clear how to construct the confidence sets. Although it is possible to approximate the confidence sets as in generalized linear bandits \citep{filippi2010parametric,provably2017li}, these approximations affect the statistical efficiency of the bandit algorithms \citep{zhang2016online}.  

To design general algorithms that do not rely on problem-specific confidence sets or posteriors, recent works proposed randomized exploration \citep{baransi2014sub,osband2015bootstrapped,kveton2019garbage,kveton2019perturbedhistory,vaswani2020old}. The key idea is to randomize the reward history of the bandit algorithms before estimating the mean rewards. The randomization strategy is general enough to apply to challenging problems, such as generalized linear bandits or neural network bandits. Bootstrapping \citep{eckles2014Thompson,osband2015bootstrapped,tang2015personalized,vaswani2018new} is one of the randomization strategies, which uses the resampled reward history for mean reward estimation. However, exploration by bootstrapping has been poorly understood theoretically. \citet{kveton2019garbage} showed that bootstrapping can suffer from linear regret in certain bandit instances and proposed to add pseudo rewards to each arm's reward history before bootstrapping. They proved that the pseudo rewards provide sufficient variance for exploration. \citet{baekjin2019on,kveton2019perturbedhistory} further showed that the sufficient variance can be induced by other randomization schemes, which they analyzed. Unfortunately, all the analyses rely on the right amount of external noise or pseudo rewards that match the problem instances. In real-world problems, however, we often do not have prior knowledge of the variance of the reward distributions. Thus the external noise and pseudo rewards are hard to design.

In this work, we propose a general randomized exploration strategy without adding any external noise or pseudo rewards. Specifically, we take advantage of the randomness in the agent's past observed rewards from all arms for exploration. In each round, the learning agent adds to each arm's history the rewards sampled from past observations of all the arms, and pulls the arm with the highest estimated mean reward based on the perturbed histories. We call the resulting algorithm \core, meaning \textbf{C}apitalizing \textbf{O}n \textbf{Re}wards. As \core only relies on past observed rewards, its exploration is \emph{data-dependent}. With a well designed sampling strategy, the observed rewards from all arms provide enough variance for exploration, without the need of knowing the actual reward distributions of the arms. Thus the exploration adapts to different problems without parameter tuning. This is a significant advantage in real-world applications, where we often have no knowledge of the actual reward distributions.

We make the following contributions. First, we propose a randomized exploration strategy that does not rely on any external noise. We show that the new algorithm \core ensures proper variance for exploration by sampling from the past observed rewards, agnostic to the variance of reward distributions. Second, we analyze \core in a linear Gaussian bandit and derive $\tilde O(d\sqrt{n\log K})$ gap-free bounds on its $n$-round regret, where $d$ is the dimension of feature vectors and $K$ is the number of arms. Although we assume Gaussian noise in the analysis, we observe empirically that \core works well when the reward noise is not Gaussian and varies significantly across the arms. Finally, we conduct comprehensive experiments on both synthetic and real-world problems that demonstrate the effectiveness of \core.

\section{Setting}
\label{sec:setting}

We use the following notation throughout the paper. The set $\{1,2,...,n\}$ is denoted by $[n]$. We denote by $u \oplus v$ the concatenation of vectors $u$ and $v$. 
% For an event $E$, $\mathbbm{1}\{E\}=1$ if $E$ occurs and $\mathbbm{1}\{E\}=0$ otherwise. A random variable $X$ is $\rho^2$-sub-Gaussian if $\E{\exp(\lambda(X-\E X))}\leq \exp(\lambda^2\rho^2/2)$ for any $\lambda>0$. 
We use $I_d$ to denote a $d\times d $ identity matrix, and use $\tilde O$ as the big-$O$ notation up to polylogarithmic factors in $n$.

% In this paper, we focus on regret minimization in a stochastic $K$-armed bandit \citep{Lai1985asymptotically, lattimore2020bandit}. This is an online learning problem where the learning agent pulls $K$ arms in $n$ rounds. In round $t\in[n]$, the agent pulls arm $I_t\in[K]$ and receives its noisy reward. We denote the reward of arm $i$ in round $t$ by $Y_{i,t}$, which is drawn i.i.d.\ from the reward distribution of arm $i$, $P_i$, with mean $\mu_i\in[0,1]$. The agent does not know the mean rewards of the arms in advance and learns them by pulling the arms. The goal of the agent is to maximize its \emph{expected} cumulative reward in $n$ rounds. 

A \emph{stochastic linear bandit} \citep{paat2008linearly,yasin2011improved} is an online learning problem where the learning agent sequentially pulls $K$ arms in $n$-rounds and each arm is associated with a $d$-dimensional feature vector. We denote $x_i\in\realset^d$ as the feature vector of arm $i\in [K]$ and $\theta_*\in\realset^d$ as the unknown parameter vector. The \emph{reward} of arm $i$ in round $t\in [n]$, $Y_{i,t}$, is drawn i.i.d.\ from the reward distribution of arm $i$, $P_i$, with mean $\mu_i=x_i\T\theta_*$. 
% \todob{Say that $Y_{i,t}$ are $\sigma^2$-sub-Gaussian. This is needed in \cref{sec:algorithm}, where we use $\sigma^2$ without ever mentioning it before.} 
In round $t$, the learning agent pulls arm $I_t\in[K]$ and receives the reward $Y_{I_t,t}$. To have a more compact notation, we denote $X_t=x_{I_t}$ and $Y_t=Y_{I_t,t}$ as the feature vector of the pulled arm in round $t$ and its observed reward. The agent does not know the mean rewards or the parameter vector in advance and learns them by pulling the arms. The goal of the agent is to maximize its \emph{expected cumulative reward} in $n$ rounds. In particular, when $x_i$ is a $K$-dimensional one-hot vector with $x_i = e_i$, $i\in [K]$, and $\theta_*$ is a vector of $K$ mean rewards, the linear bandit reduces to a \emph{multi-armed bandit} \citep{Lai1985asymptotically,lattimore2020bandit}. 

Without loss of generality, we assume that arm $1$ is optimal, meaning $\mu_1>\max_{i>1}\mu_i$. We denote by $\Delta_i=\mu_1-\mu_i$ the \emph{gap} of arm $i$, which is the difference between the mean rewards of arms $1$ and $i$. Maximizing the expected $n$-round reward is equivalent to minimizing the \emph{expected $n$-round regret}, which is defined as
\begin{equation*}
    R(n) = \sum_{i=2}^K\Delta_i\E{\sum_{t=1}^n\mathbbm 1\{I_t=i\}}.
\end{equation*}
We make the following standard assumptions in this setting. First, the mean reward $\mu_i=x_i\T\theta_*$ for any arm $i\in[K]$ is bounded, and without loss of generality, we assume that it is in $[0, 1]$. Second, the feature vector of the last $d$ arms are a basis in $\realset^d$. This is without loss of generality, as the arms can always be ordered to satisfy this. 
% When $x_i$ is a $K$-dimensional one-hot vector with $x_i = e_i$, $i\in [K]$, and $\theta_*$ is a vector of $K$ mean rewards, the linear bandit degrades to a \emph{multi-armed bandit}. 
% \todob{Be more precise. You never say what the other entries are. A standard notation for the $i$-th vector in the standard Euclidean basis is $e_i$. So say $x_i = e_i$ and we are done.} 

\section{Capitalizing on Rewards in Bandit Exploration}
\label{sec:algorithm}
In this section, we introduce the new algorithm Capitalizing On Rewards (\core). We first illustrate key ideas of \core and discuss how it works in \cref{sec:idea}. In \cref{sec:lincore}, we instantiate the algorithm in a stochastic linear bandit. To be more specific, in the rest of the paper, we use \core to refer to the algorithm applied in a multi-armed bandit, and use \lincore to represent the algorithm in a linear bandit. 

\subsection{Key Ideas and Informal Justification of \core}
\label{sec:idea}
The principle of \core is to utilize the variance in the past observed rewards to incentivize exploration. We first discuss \core in a simple multi-armed bandit to illustrate how it works. Specifically, when estimating the mean reward of arm $i$ in round $t$, \core first perturb each reward of arm $i$ with a reward sampled from all observed rewards in the past $t-1$ rounds. Then the mean of arm $i$ is estimated based on its perturbed rewards. Thus if there is sufficient variance in the past observed rewards, \core is able to overestimate the mean rewards of arms to achieve optimism. 

To be more concrete, we make an analogy between \core and \ts. For example, in a \emph{Gaussian bandit}, adding additive noise to the mean reward estimate is equivalent to posterior sampling. In particular, fix arm $i$ and the number of its pulls $s$. Let $\mu_i\sim\mathcal N(\mu_0,\sigma^2)$ be the mean reward of arm $i$, where $\mathcal N(\mu_0,\sigma^2)$ is the Gaussian prior in \ts and $\sigma^2$ is the variance of the arm's reward distribution. Let $(Y_{i,\ell})_{\ell=1}^s \sim \mathcal N(\mu_i,\sigma^2)$ be $s$ i.i.d.\ noisy observations of $\mu_i$. Then the posterior distribution of $\mu_i$ conditioned on $(Y_{i,\ell})_{\ell=1}^s$ is 
\begin{equation}
\label{eq:posterior}
    \mathcal N\Bigg(\frac{\mu_0+\sum_{\ell=1}^sY_{i,\ell}}{s+1}, \frac{\sigma^2}{s+1}\Bigg).
\end{equation}
It is well known that sampling from this distribution in \ts leads to near-optimal regret \citep{agrawal2013furtherTS}. From another perspective, sampling a mean reward of arm $i$ as above is equivalent to adding i.i.d.\ Gaussian noise to $\mu_0$ and each reward in $(Y_{i,\ell})_{\ell=1}^s$, and then taking the average \citep{kveton2019perturbedhistory}. Specifically,
\begin{equation*}
    \frac{\mu_0+Z_0 + \sum_{\ell=1}^{s}(Y_{i,\ell}+Z_\ell)}{s+1}
\end{equation*}
is a sample from distribution \eqref{eq:posterior} for $(Z_\ell)_{\ell=0}^s\sim\mathcal N(0,\sigma^2)$.

However, in practice, $\sigma^2$ depends on the specific problem instance and is unknown. Thus the variance of $(Z_\ell)_{\ell=0}^s$ needs to be carefully tuned to match $\sigma^2$. The key insight in \core is that the exact value of $\sigma^2$ does not have to be known. Instead of sampling the noise from a given distribution, \core samples $(Z_{l})_{\ell=0}^s$ from a reward pool, which is composed of previously observed rewards of all arms. Then \core adds sampled rewards to each reward of arm $i$ for mean reward estimate. As we show in \cref{sec:analysis}, after an initialization period of $\frac{4\log n}{z-1-\log z}+1$ rounds, for any $z \in (0, 1)$, the empirical variance of the observed rewards so far is at least $z\sigma^2/2$ with a high probability. Thus, after scaling the rewards by $\alpha$ to construct the reward pool, the variance of each i.i.d.\ sampled reward is greater than $\alpha^2z\sigma^2/2$. This is at least $\sigma^2$ for $\alpha^2 z > 2$, and can be achieved without knowing $\sigma^2$. 

% \todob{The algorithm is hard to read, mainly because it is unclear what the reward pool is. Here is my suggestion on how to fix this:

% 1) Turn \cref{sec:informal} into a separate introductory section / subsection. Call it "Key Idea" or something like that. The strength of the section is that it is informal, to the point that we do not need to refer to LinCORe at all. Explain how the variance in observed rewards, which vary both because of reward noise and their means, is an overestimate of $\sigma^2$ and thus sufficient for exploration.

% 2) When you describe LinCORe later, be careful when we introduce $\alpha$. The current feeling is that $\alpha$ is needed because the reward pool is poorly designed. This is not the case. We need $\alpha > 1$ for analysis and typically use $\alpha < 1$ in experiments because LinCORe is too optimistic. But LinCORe for $\alpha = 1$ is expected to perform well in practice, based on our informal justification. Say it so that the reviewer understands our choices.}

\subsection{Capitalizing on Rewards in a Stochastic Linear Bandit}
\label{sec:lincore}
We present the algorithm in a stochastic linear bandit (\lincore) in \cref{alg:lincore}, as it is a more general setting than a multi-armed bandit. In round $t$, \lincore first constructs a \emph{reward pool} $\mathcal{R}_t$ from all the past $t-1$ observed rewards. To achieve optimism in the mean reward estimate in round $t$, each reward $Y_\ell$ observed from a pulled arm with feature vector $X_\ell$ is perturbed by a randomly sampled reward $Z_\ell$ from $\mathcal{R}_t$ to fit a linear model (line 11), 
% \todob{It is absolutely not clear what the reward pool is. See my suggestions at the end of the section on how to fix this.} \todob{It is not clear why $\alpha$ is needed.} 
\begin{equation}
    \tilde{\theta}_t \leftarrow G_t^{-1}\sum_{\ell=1}^{t-1}X_\ell\Big[Y_\ell+ Z_\ell\Big], 
\end{equation}
where
\begin{equation}
    \quad G_t \leftarrow \sum_{\ell=1}^{t-1} X_\ell X_\ell^\top + \lambda I_d
\end{equation}
is the sample covariance matrix up to round $t$ and $\lambda>0$ is the regularization parameter. $\mathcal (Z_\ell)_{\ell=1}^{t-1}$ are i.i.d.\ rewards freshly sampled in each round from $\mathcal{R}_t$. The estimate of the mean reward of arm $i$ is $x_i\T \hat \theta_t$. The arm with the highest reward estimate is pulled. This is similar to Thompson sampling \citep{thompson1933on,agrawal2012thompson,abeille2016linear} and perturbed history exploration  \citep{kveton2019perturbedhistory,kveton2020linphe} in linear bandits, which add noise to the parameter estimate. However, \lincore does not depend on any posterior variance or external pseudo rewards for exploration, and instead only relies on randomness in the agent's own reward history. 
\begin{algorithm}[t]
\caption{Capitalizing on Rewards in a stochastic linear bandit (\lincore)}
\hspace*{\algorithmicindent}\;\textbf{Input}: Initial variance ratio $z\in(0,1)$, sample \\ \hspace*{\algorithmicindent}\; scale ratio $\alpha\in\mathbb R^+$, number of rounds $n$

\begin{algorithmic}[1]
\label{alg:lincore}
% \STATE \textbf{Initialize}: Reward pool $\mathcal R\leftarrow ()$
\FOR{$t=1,2,...,n$}
\IF{$t\leq \max\big\{d, \frac{4\log n}{(z-1-\log z)}+1\big\}$}
\STATE $I_t\leftarrow t \Mod K$
% \STATE Pull arm $I_t$ and get reward $Y_{I_t, t}$
\ELSE
% \STATE $\mu=\frac{1}{|\mathcal{R}|}\sum_{y\in\mathcal{R}}y,\; \mathcal{\tilde R} \leftarrow()$
% \FOR{$y\,\text{in}\,\mathcal{R}$}
% \STATE $y\leftarrow y-\mu$, $\mathcal{\tilde R}\oplus(y,-y)$
% \ENDFOR
        \STATE $\mathcal{R}_t \gets ()$
        \STATE $\mu(\mathcal{R}_t) = \frac{1}{t - 1} \sum_{\ell = 1}^{t - 1} Y_{\ell}$
        \FOR{$\ell = 1, \dots, t - 1$}
            \STATE $\mathcal{R}_t \gets \mathcal{R}_t \oplus (\alpha (Y_{\ell} - \mu(\mathcal{R}_t)), \, \alpha(\mu(\mathcal{R}_t) - Y_{\ell}))$
        \ENDFOR
\STATE $(Z_\ell)_{\ell=1}^{t-1} \leftarrow$ Sample $t-1$ i.i.d.\ rewards from $\mathcal{R}_t$
\STATE $G_t \leftarrow \sum_{\ell=1}^{t-1} X_\ell X_\ell^\top + \lambda I_d$
\STATE $\tilde{\theta}_t \leftarrow G_t^{-1}\sum_{\ell=1}^{t-1}X_\ell\Big[Y_\ell+ Z_\ell\Big]$
% \todob{I suggest that we fold $\alpha$ into the reward pool to simplify exposition. In particular, $\mathcal{R}_t \gets \mathcal{R}_t \oplus (\alpha (Y_{\ell} - \mu(\mathcal{R}_t)), \, \alpha(\mu(\mathcal{R}_t) - Y_{\ell}))$.}
\STATE $I_t\leftarrow \arg\max_{i\in [K]}x_i^\top\hat\theta_t$
\ENDIF
\STATE Pull arm $I_t$ and get reward $Y_{t}$ 
% \leftarrow Y_{I_t,t}
% \STATE $\mathcal{R} \leftarrow \mathcal{R} \oplus (Y_{t})$
\ENDFOR
\end{algorithmic}
\end{algorithm}

Specifically, in lines 1-3 of \cref{alg:lincore}, we initialize \lincore by pulling arms sequentially for the first $\max\{d,\frac{4\log n}{z-1-\log z}+1\}$ rounds, where $z\in{(0,1)}$ is a tunable parameter that determines the initial variance in the reward pool. After initialization, in each round $t$, \lincore processes the past $t-1$ rewards and creates a new reward pool $\mathcal{R}_t$ in lines 5-8. \lincore scale the rewards by $\alpha$ to guarantee sufficient variance in $\mathcal R_t$ for exploration, as suggested in \cref{sec:idea}. Besides, the processed rewards in $\mathcal{R}_t$ are centered to have zero mean and each reward $y$ has its symmetric reward $-y$ around zero in the pool. This additional processing is only to simplify the theoretical analysis in \cref{sec:analysis}. It does not change the variance of samples from the reward pool and \lincore performs in practice similarly without it. \lincore then samples $t-1$ i.i.d.\ rewards from $\mathcal{R}_t$ (line 9). To get the parameter estimate $\hat \theta_t$, \lincore perturb each observed reward by a sampled reward from $\mathcal{R}_t$ to fit a linear model (lines 10-11). Finally, \lincore pulls arm $I_t$ with highest mean reward estimate from $\tilde\theta_t$ and observe its reward $Y_t$. It is important to note that \cref{alg:lincore} is only an instance of the proposed general randomization strategy in a linear bandit setting. The parameter estimation in lines 10-12 (\cref{alg:lincore}) can be replaced by any other estimator, such as a neural network, to get more general algorithms. Here we choose to show the linear case rather than a general case to be more concrete for reproducibility. Besides, when  feature vectors are one-hot vectors with $x_i=e_i$, \cref{alg:lincore} corresponds to \core in a multi-armed bandit, which is essentially using the average of each arm's perturbed rewards as the mean reward estimate.  
 
The exploration in \lincore arises from the variance in $\mathcal R_t$. For example, if the reward distributions of all arms are Gaussian with variance of $\sigma^2$, 
% \todob{This needs to be introduced earlier. Here we need Gaussianity, which is of course sub-Gaussian.} 
we want a comparable variance in $\mathcal R_t$, so that the sampled rewards from $\mathcal R_t$ can offset unfavorable reward histories. To achieve this, \lincore initially pulls arms sequentially $\max\{d,\frac{4\log n}{z-1-\log z}+1\}$ times, to accumulate observations. We prove in \cref{sec:variance-analysis} that after this initialization, the empirical variance of observed rewards is at least $z\sigma^2/2$ with a high probability. However, $z\sigma^2/2$ may not be sufficient for effective exploration. Once $z$ is fixed, the \emph{scale ratio} $\alpha$ dictates the multiplicative factor of the variance of each sampled reward in the reward pool, and thus controls the trade-off between exploration and exploitation. Larger $\alpha$ leads to more exploration. More importantly, the variance in $\mathcal R_t$ is at least $\alpha^2 z/2$ of that of the reward distributions. So it is automatically adapted to the problem. 
% Although our theoretical analysis indicates that $\alpha^2 z/2$ should  sufficient exploration, this might be too optimistic and conservative in practice. We show in \cref{sec:experiments} that typically $\alpha^2 z/2 < 1$ can already perform well in practice. 

% Although we show in \cref{sec:idea} that $\alpha^2 z/2 > 1$ is needed, \todob{We never showed that this is really needed. Rephrase.} this might be too optimistic and conservative in practice. Typically $\alpha^2 z/2 < 1$ can already perform well in practice. 

% Next we provide an insight on how \core achieves sublinear regret when $\alpha^2z>2$, without knowing the value of $\sigma^2$.

\section{Regret Analysis}
\label{sec:analysis}

We analyze the regret of \lincore in the case of Gaussian rewards, where the rewards of arm $i$ are sampled i.i.d.\ from a Gaussian distribution $Y_{i, t} \sim \mathcal N(\mu_i, \sigma^2)$ for all $i \in [K]$ and $t \in [n]$, and $\mu_i \in [0, 1]$. The variance of reward distributions is $\sigma^2$, identical for all arms. Based on this setting, we derive the following gap-free bound on the $n$-round regret of \lincore.
\begin{restatable}[]{theorem}{lincoreregret}
\label{thm:lincore regret bound} For any $1 / 2 \leq z < 1$, $4 \sqrt{\sigma^2 \log n} \geq 1$, and $n \geq 24$, the expected $n$-round regret of \lincore is
\begin{align*}
  R(n)
  = \tilde{O}(d \sqrt{n \log K})
\end{align*}
\end{restatable}
for $\alpha = O(\sqrt{z^{-1} d \log n})$. We provide the detailed proof in \cref{sec:regret bound}.

\subsection{Discussion}

The regret of \lincore is $\tilde{O}(d \sqrt{n \log K})$ (\cref{thm:lincore regret bound}), where $d$ is the number of features and $K$ is the number of arms. This is on the same order as the regret bound of \linphe \citep{kveton2020linphe}, a state-of-the-art randomized algorithm for linear bandits. In the infinite arm setting, \citet{abeille2016linear} proved that the regret of \lints is $\tilde{O}(d^\frac{3}{2} \sqrt{n})$, which we also match. Specifically, if the space of arms was discretized on an $\varepsilon$-grid, the number of arms would be $K = \varepsilon^{- d}$ and $\sqrt{\log K} = \sqrt{d \log(1 / \varepsilon)}$.

The key idea in our analysis is to inflate $\alpha$ in the reward pool to achieve optimism. In linear bandits, this idea can be traced to \citet{agrawal2012thompson}. Roughly speaking, $\alpha = O(\sqrt{z^{-1} d \log n})$. This setting is too conservative in practice. Therefore, we experiment with less conservative settings in \cref{sec:experiments}.

The main challenge of our analysis is to analyze the behavior of realized rewards in the reward pool. In \cref{lemma:init-variance}, we show that the rewards have sufficient variance. We bound their magnitude in \cref{lemma:bound_reward}. The rest of our analysis follows the outline of \linphe \citep{kveton2020linphe}, which we generalize from Bernoulli pertubations to those in \cref{sec:variance-analysis}.

\subsection{Reward Pool}
\label{sec:variance-analysis}

The exploration in \lincore is enabled by the variance of sampled rewards from the reward pool $\mathcal R_t$. In this section, we analyze the variance of sampling i.i.d.\ rewards from $\mathcal R_t$, which lays the foundation for the theoretical analysis of \lincore. We use $\sigma^2(\mathcal R_t)$ to represent the variance of one i.i.d.\ sampled reward from $\mathcal R_t$. Specifically, the rewards in $\mathcal R_t$ are simple transformations of all the past $t-1$ observed rewards $(Y_{I_\ell,\ell})_{\ell=1}^{t-1}$ (lines 6-8 in \cref{alg:lincore}). $\sigma^2(\mathcal R_t)$ is algebraically equivalent to the variance of one sampled reward from $(Y_{I_\ell,\ell})_{\ell=1}^{t-1}$ scaled by $\alpha^2$, 
\begin{equation}
    \sigma^2(\mathcal R_t) = \frac{1}{|\mathcal R_t|} \sum_{y\in\mathcal R_t} y^2 = \frac{\alpha^2}{t - 1} \sum_{\ell = 1}^{t - 1} (Y_{I_\ell, \ell} - \mu(\mathcal{R}_t))^2,
\end{equation}
where $\mu(\mathcal{R}_t)$ is the mean of all past rewards observed by the learning agent, as defined in line 6 of \cref{alg:lincore}. Thus a sampled reward from $\mathcal R_t$ can provide the variance of $\sigma^2(\mathcal R_t)$. We characterize $\sigma^2(\mathcal R_t)$ by the following two lemmas, which are proved in \cref{sec:reward pool lemmas}.

\begin{restatable}[]{lemma}{initialvariance}
\label{lemma:init-variance}
For any $n\geq 2$ and $z\in(0,1)$, $\sigma^2(\mathcal R_t)\geq \frac{\alpha^2 z}{2}\sigma^2$ with probability of at least $1-\frac{1}{n}$, jointly for all rounds $t > \frac{4\log n}{z-1-\log z}+1$.
\end{restatable}
\cref{lemma:init-variance} states that when there are enough rewards in $\mathcal R_t$ after the initialization, the variance of sampling a reward from $\mathcal R_t$ is $\Omega(\sigma^2)$ with a high probability, which provides the variance needed for exploration. On the other hand, the variance should not be too large, which would hurt the convergence of mean reward estimates. \cref{lemma:bound_reward} shows that the rewards in $\mathcal R_t$ are bounded with high probability.
\begin{restatable}[]{lemma}{boundreward}
\label{lemma:bound_reward}
    For any $n\geq 2$ and $t\leq n$, with probability of at least $1-\frac{1}{n}$, the absolute values of the rewards in reward pool $\mathcal R_t$ are bounded by $\alpha(4\sqrt{\sigma^2\log(n)}+1)$. 
\end{restatable}
% The proof is based on the sub-Gaussianity of bounded variables. In the regret analysis, we only consider the case of $16\sigma^2\log n$-sub-Gaussian of sampling from $\mathcal R_t$ for the worst case analysis, since $1$-sub-Gaussian is independent from $\sigma^2$ and the number of rounds $n$, which only gives us a better regret bound. \todob{I cannot follow the above line of thought. Let's discuss this.}
% Based on \cref{lemma:init-variance,lemma:bound_variance}, we can reach the following conclusion with the union bound. When $n\geq 2$, $K\geq 2$ and $\max\{K,\frac{2\log n}{z-1-\log z}+1\}\leq t-1 < n$, with probability of at least $1-2/n$, sampling one i.i.d.\ reward from $\mathcal R_t$ has variance $\sigma_r^2(\mathcal R_t)\geq z\sigma^2/2$, and follows a $\max\{1,16\sigma^2\log n\}$-sub-Gaussian distribution. In particular, the lower bound on $\sigma^2(\mathcal R_t)$ ensures the overestimate of the mean reward estimate for exploration. The sub-Gaussianity of sampling from $\mathcal R_t$ guarantees that as long as an arm is pulled enough times, the mean reward estimate of it will converge to the expected mean for exploitation, despite the added rewards from $\mathcal R_t$. \cref{lemma:init-variance,lemma:bound_variance} provide the justification of using the agent's past observed rewards for effective exploration in \core, and are applied throughout the proof of \cref{theorem:gap-dependent-bound} in \cref{appendix:prove-bound}. 
In particular, in \cref{lemma:init-variance}, the lower bound on $\sigma^2(\mathcal R_t)$ ensures the overestimate of the mean reward estimate for exploration. The bound of the scale of sampled rewards from $\mathcal R_t$ in \cref{lemma:bound_reward} indicates the convergence of the mean reward estimates. \cref{lemma:init-variance,lemma:bound_reward} provide the justification of using the agent's past observed rewards for effective exploration in \lincore, and are applied throughout the proof of \cref{thm:lincore regret bound} in \cref{sec:regret bound}. 

\begin{figure*}[h]
    \includegraphics[width=0.32\textwidth]{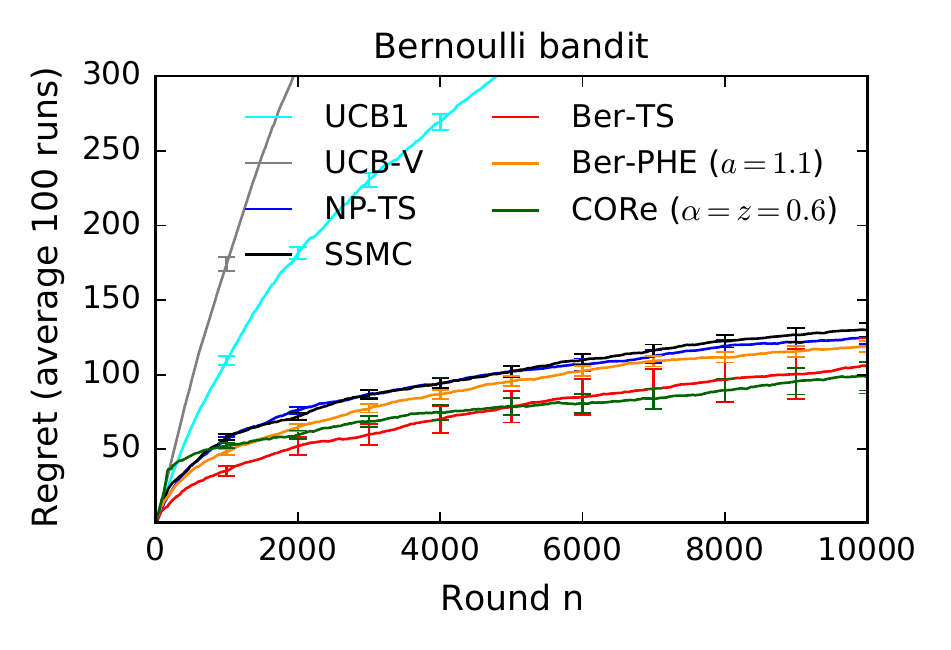}
    \includegraphics[width=0.32\textwidth]{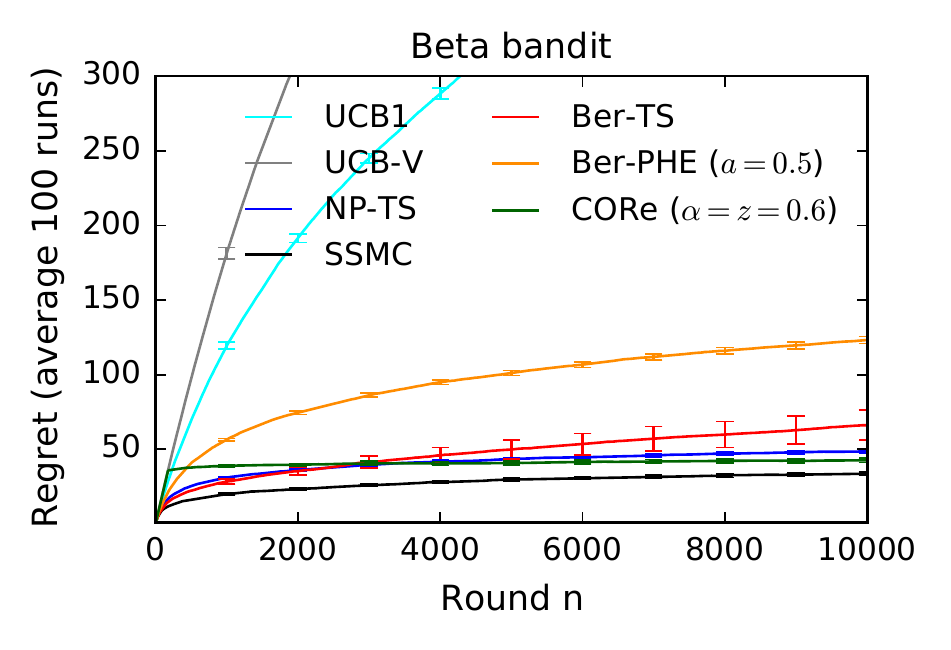}
    \includegraphics[width=0.32\textwidth]{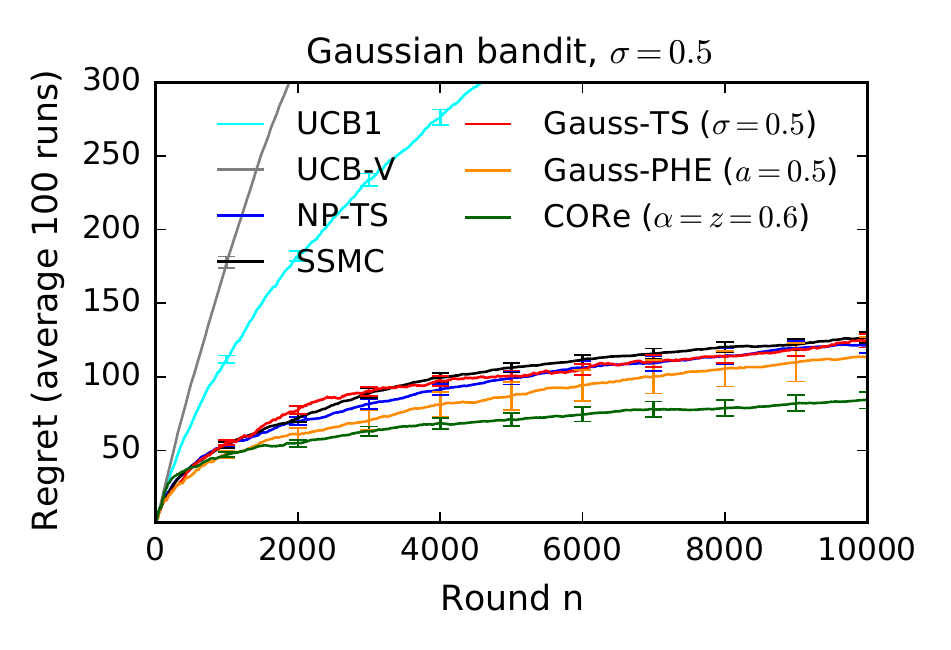}
    \caption{Comparison between \core and the baselines in Bernoulli, Beta and Gaussian multi-armed bandits. $x$-axis is the number of rounds. $y$-axis is the cumulative regret (lower the better) and is displayed on log scale to include all baselines. All results are averaged over 100 randomly chosen problems and error bars represent the standard deviation over the runs.}
    \label{fig:compare-baselines}
\end{figure*}

\section{Experiments}
\label{sec:experiments}

In this section, we evaluate our proposed algorithm empirically in both multi-armed bandits and linear bandits. In all experiments, we keep the notation \core to denote the proposed algorithm in a multi-armed bandit setting, and \lincore in the linear case. We compare it with several state-of-the-art baselines and show how it adapts to different problems without parameter tuning. In \cref{sec:multiarmed-bandit}, we evaluate \core in multi-armed bandit problems. We experiment with \lincore in linear bandit problems in \cref{sec:linear-bandit} and investigate the robustness of its parameters in \cref{sec:tuning}. Finally, we generalize \core to a learning to rank problem to evaluate its performance in real-world problems. 

\subsection{Multi-Armed Bandit}
\label{sec:multiarmed-bandit}
We evaluate \core in three classes of multi-armed bandit problems. The first class is Bernoulli bandits where $P_i=\mathrm{Ber}(\mu_i)$. The second class is beta bandits where $P_i=\mathrm{Beta}(v\mu_i,v(1-\mu_i))$ with $v=4$. The third class is Gaussian bandits where $P_i=\mathcal{N}(\mu_i,\sigma^2)$ with $\sigma=0.5$. Each bandit problem has $K=10$ arms and the mean rewards are chosen uniformly at random from $[0.25, 0.75]$. The horizon of each experiment is $n=10,000$ rounds. We experiment with $100$ randomly chosen problems in each class and report the average regret. 

% \todob{$8$-page papers are too short for sub-sub-sections.} 
We compare \core to six baselines: \ucb \citep{auer2002finite}, \ucbv \citep{audibert2009ucbv}, \ts \citep{agrawal2013furtherTS}, \phe \citep{kveton2019perturbedhistory}, \npts \citep{pmlr-v117-riou20a} and \ssmc \citep{chan2019multiarmed}. 
%For Bernoulli and Beta bandits, we use Bernoulli \ts and add Bernoulli pseudo rewards in \phe. For Gaussian bandits, we use Gaussian \ts and add Gaussian pseudo rewards in \phe. 
\ucbv can estimate the variance of the reward distribution based on the observed rewards, which automatically adapts to the variance. \npts and \ssmc are two non-parametric solutions proposed in the multi-armed bandit setting. In particular, \npts is a non-parametric randomized algorithm. At each step, it computes an average of the observed rewards
with random weights. \ssmc is a non-parametric arm allocation procedure inspired by sub-sampling approaches \cite{baransi2014sub}. For \ts, we use Bernoulli \ts (Ber-\ts) with a $\mathrm{Beta}(1,1)$ prior for Bernoulli and beta bandits. We use Gaussian \ts (Gauss-\ts) with a $\mathcal N(0.5, \sigma^2)$ prior for Gaussian bandits \citep{agrawal2013furtherTS}, where the parameter $\sigma$ is set to match the variance of the actual reward distribution. \phe belongs to the same class of bandit algorithms as \core that randomize the reward history for exploration. We do not further include Giro \citep{kveton2019garbage} as \phe explores similarly but in a more efficient way. We add Bernoulli pseudo rewards in \phe (Ber-\phe) in Bernoulli and beta bandits and set the parameter $a$ to values that achieve the best performance as reported in \citep{kveton2019perturbedhistory}. For Gaussian bandit, we add Gaussian pseudo rewards (Gauss-\phe) as suggested in the paper. We set the standard deviation of the Gaussian pseudo rewards to $0.5$ and tune parameter $a$ in the range of $[0.1, 2]$ with step size of $0.1$. 
 For \core, we fix the parameters $\alpha=z=0.6$ for all three classes of problems. 
% \todob{Do we need the Ber and Gauss prefixes? Or do we just say how we implement \ts and \phe in the three problem classes?}
% \todon{Good question. I was also wondering about this, but decided to make things clearer. Do you know any tool that can directly edit the legends of such python generated figures? My only concern is reruning the experiments will be a bit time-consuming since I did not save the results...} If not specified otherwise, we fix the parameters of \core as $\alpha=z=0.6$ by default, but tune the parameters of other baselines in different problems as needed. 

Our results are reported in \cref{fig:compare-baselines}. We show the cumulative regret as a function of the number of rounds. \core achieves strong empirical performance that is comparable to or better than all the baselines. In particular, \core outperforms \ucb and \ucbv in all three classes of bandit problems. Although \ucbv estimates the variance of observed rewards to explore, it is too conservative and performs poorly in practice. \ts and \phe can have similar performance as \core, but the variance of the posterior (parameter $\sigma$) in \ts and the perturbation scale in \phe (parameter $a$) are tuned based on the knowledge of the specific bandit problems, which is usually not accessible in real-world scenarios. In contrast, \core consistently performs well in different problems without tuning the parameters. This is a significant advantage in real-world applications when the reward distribution is unknown. \npts and \ssmc also achieve strong performance in multi-armed bandits, but they do not generalize to structured problems as \core does. 

% \begin{figure*}
%     \centering
%     \begin{subfigure}[b]{0.32\textwidth}
%          \centering
%          \includegraphics[width=\textwidth]{new_figures/compare_baselines_gauss_cross_9_15_3_51.pdf}
%          \caption{Easy problem}
%          \label{fig:easy}
%     \end{subfigure}
%     \begin{subfigure}[b]{0.32\textwidth}
%          \centering
%          \includegraphics[width=\textwidth]{new_figures/compare_baselines_gauss_cross1_9_15_4_9.pdf}
%          \caption{Hard problem}
%          \label{fig:hard}
%     \end{subfigure}
%     \begin{subfigure}[b]{0.32\textwidth}
%          \centering
%          \includegraphics[width=\textwidth]{new_figures/core_tune_gauss_9_15_21_23.pdf}
%          \caption{Tuning parameters of \core}
%          \label{fig:tune}
%     \end{subfigure}
%     \caption{In the first two figures, we compare \core, \phe, and \ts in easy and hard problems. We fix the parameters of \core and tune the parameters of \phe and \ts to perform well in either the easy or the hard case. In the third figure, we tune the parameters of \core and show its final regret after 10,000 rounds. All results are averaged over 100 runs.}
%     \label{fig:compare-baselines_gauss_cross}
% \end{figure*}

\begin{figure*}
    \centering
    \includegraphics[width=0.32\textwidth]{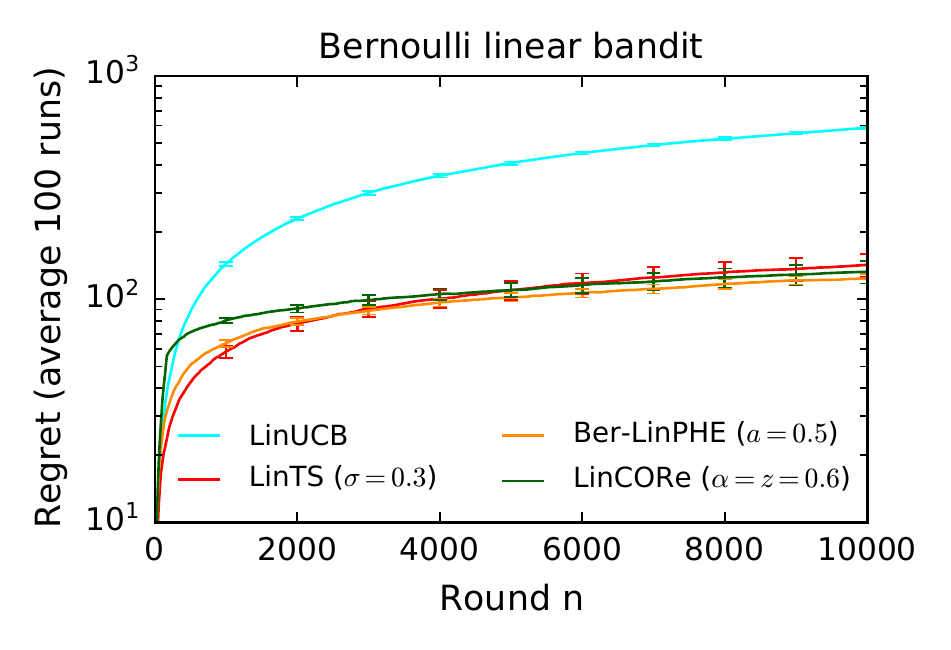}
    \includegraphics[width=0.32\textwidth]{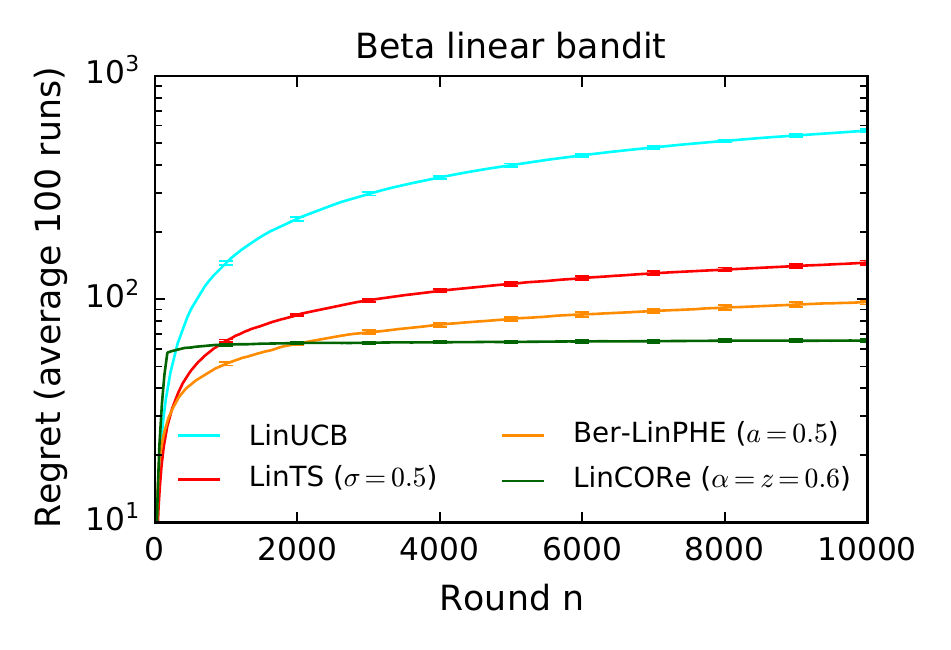}
    \includegraphics[width=0.32\textwidth]{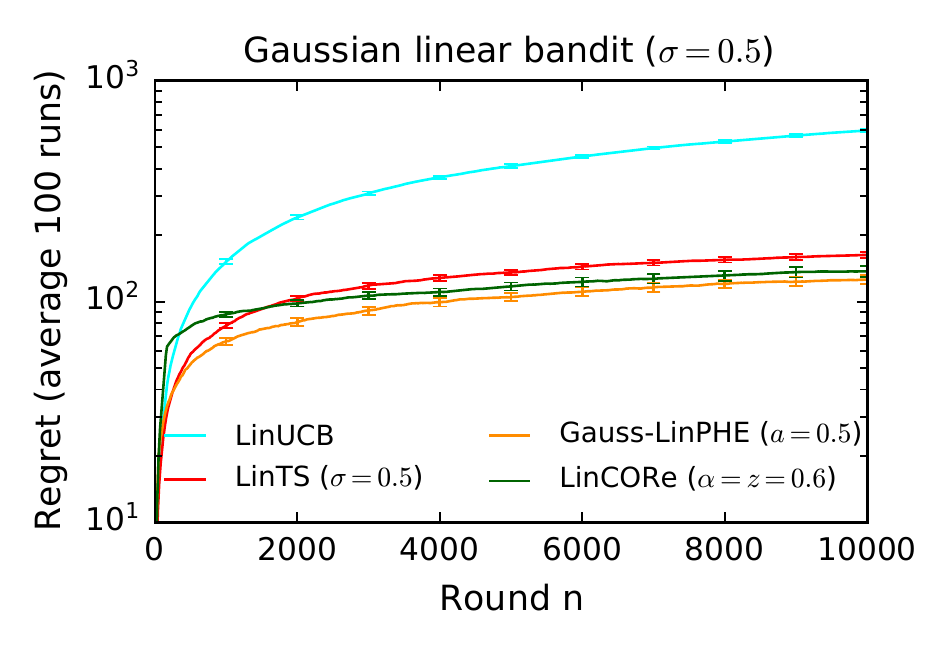}
    \caption{Comparison between \lincore and the baselines in linear bandits. The cumulative regret is displayed on log scale to include all baselines.}
    \label{fig:compare-baselines-linear}
\end{figure*}

\subsection{Linear Bandit}
\label{sec:linear-bandit}
We evaluate \lincore in several linear bandit problems. We set the number of arms to $K=50$ and the dimension of the feature vectors to $d=10$. We follow the generation of feature vectors and the parameter vector $\theta_\ast$ in \citep{kveton2020linphe} (see their Section 5.1). Following the experiments in \cref{sec:multiarmed-bandit}, we consider Bernoulli, Beta, and Gaussian reward distributions by setting the mean reward of each arm to $x_i^\top \theta_\ast \in [0, 1]$. The horizon of each experiment is $n=10,000$ rounds and we report the average results over 100 randomly chosen problems. 

We compare \lincore with \linucb \citep{yasin2011improved}, \lints \citep{agrawal2012thompson}, and \linphe \citep{kveton2020linphe}. There is no linear versions for \ucbv, \npts or \ssmc. For \linphe, we add Bernoulli pseudo rewards in Bernoulli and beta bandit (Ber-\linphe), and add Gaussian rewards in Gaussian bandit (Gauss-\linphe). The parameters for \lints and \linphe are searched in the range of $[0.1,2]$, while we still use the same parameters $\alpha=z=0.6$ for \lincore as in \cref{sec:multiarmed-bandit}. The results are shown in \cref{fig:compare-baselines-linear}. In all three classes of linear bandit problems, \lincore can achieve the best performance without tuning the parameters. Note that unlike \linucb and \lints, whose upper confidence sets and posterior need to be designed differently for multi-armed bandits and linear bandits, \lincore is simply applying the same randomization strategy to different bandit settings. Although \linphe is also a direct generalization of the multi-armed bandit setting, its perturbation from pseudo rewards depends on the knowledge of the arms' reward distribution. % \todob{The figure has Gauss-\linphe in it. We describe the algorithm as \linphe.}

\begin{figure*}
    \centering
    \begin{subfigure}[b]{0.32\textwidth}
         \centering
         \includegraphics[width=\textwidth]{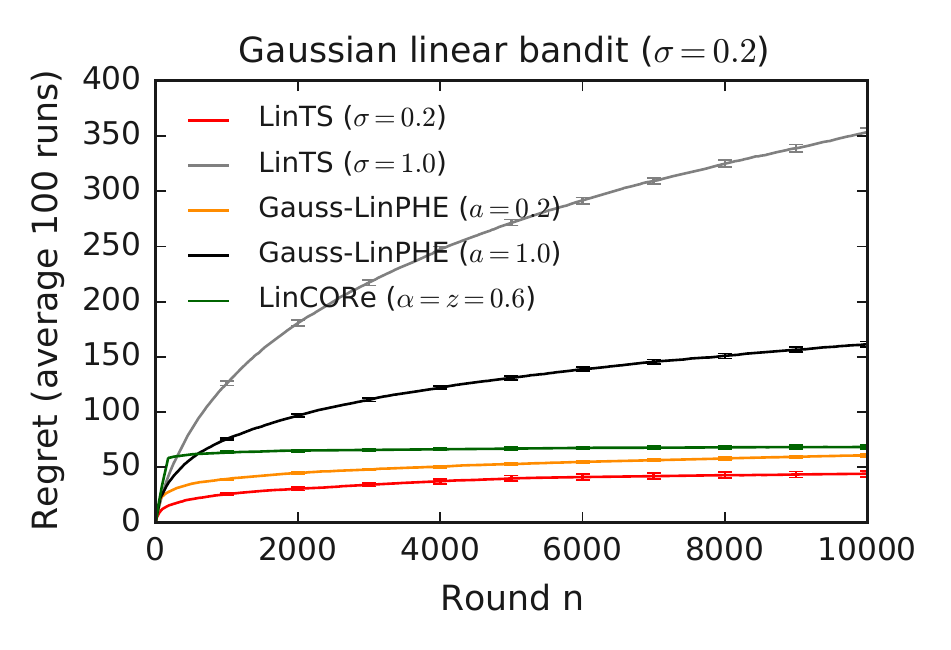}
         \caption{Easy problem}
         \label{fig:linear_easy}
    \end{subfigure}
    \begin{subfigure}[b]{0.32\textwidth}
         \centering
         \includegraphics[width=\textwidth]{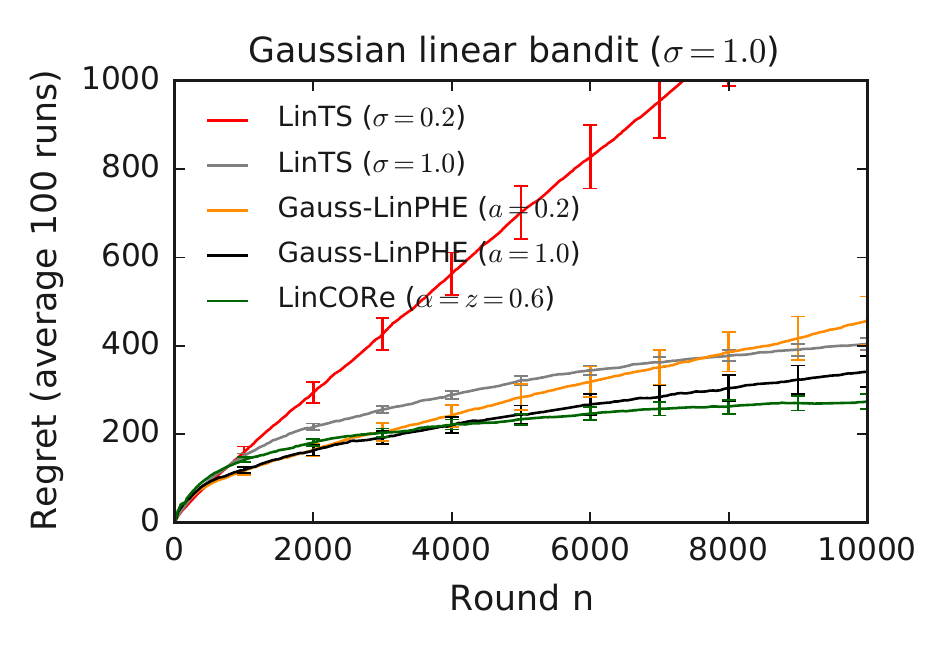}
         \caption{Hard problem}
         \label{fig:linear_hard}
    \end{subfigure}
    \begin{subfigure}[b]{0.32\textwidth}
         \centering
         \includegraphics[width=\textwidth,height=14em]{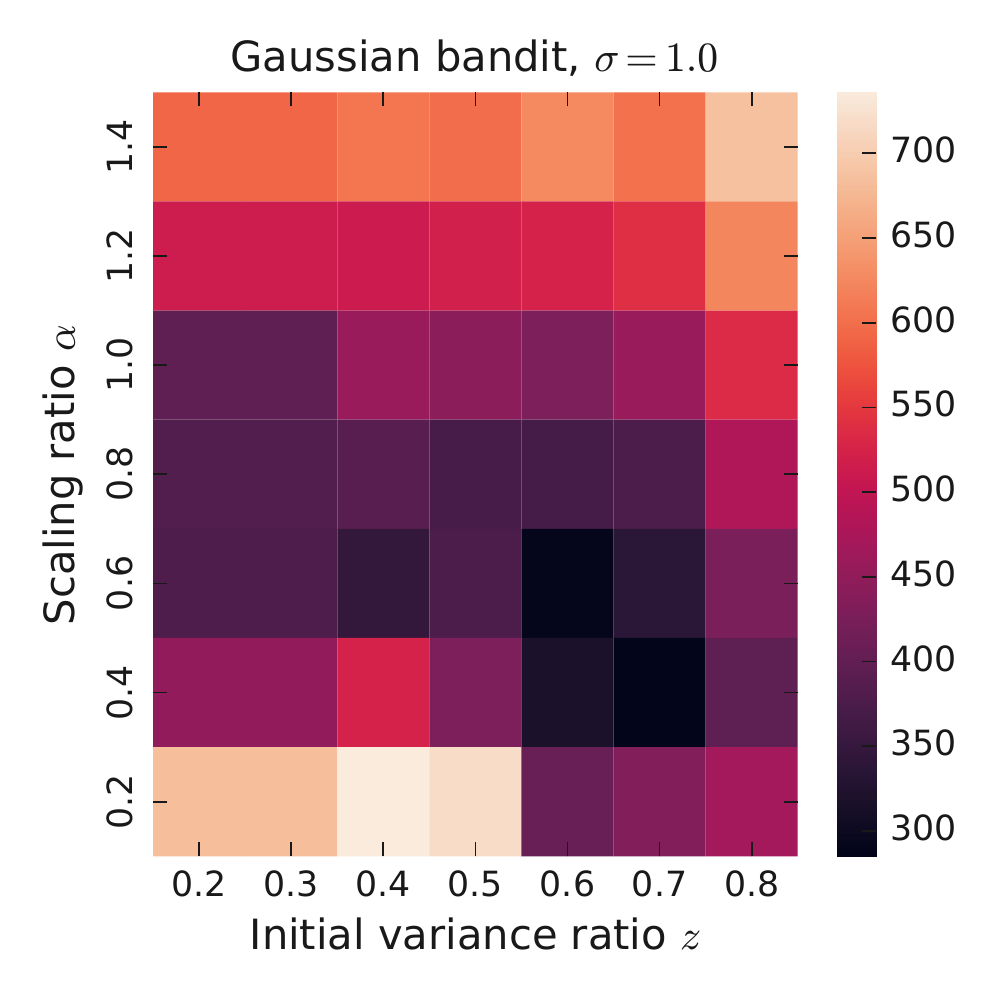}
         \caption{Tuning parameters of \lincore}
         \label{fig:linear_tune}
    \end{subfigure}
    \caption{In the first two figures, we compare \lincore, \linphe, and \lints in easy and hard problems. We fix the parameters of \lincore and tune the parameters of \linphe and \lints to perform well in either the easy or the hard case. In the third figure, we tune the parameters of \lincore and show its cumulative regret in 10,000 rounds. All results are averaged over 100 runs.}
    \label{fig:compare-baselines_gauss_cross_linear}
\end{figure*}

\subsection{Adaptation to Problem Hardness} 
\label{sec:tuning}
We further investigate how \lincore automatically adjusts its exploration in problems with different levels of hardness. Besides, we also show that \lincore works properly with a wide range of parameters. Specifically, we consider linear Gaussian bandits with different levels of variance. We set the standard deviation of the reward distributions to $\sigma=0.2$ as an easy problem, and set $\sigma=1$ as a hard problem. We compare \lincore to \lints and \linphe who achieve similar performance in \cref{sec:linear-bandit}. 
For \lints and \linphe, we use two sets of parameters for each of them, with each set specially tuned for either the easy or the hard problem. In particular, for \lints we set $\sigma$ to $0.2$ for the easy problem and $1.0$ for the hard problem that performs well in two problems correspondingly. In \linphe, we tune the parameter $a$ and set it to $0.2$ and $1.0$ for the easy and the hard problem, respectively. We still use the same fixed parameters as in \cref{sec:multiarmed-bandit,sec:linear-bandit} in \lincore for both problems. As shown in \cref{fig:linear_easy,fig:linear_hard}, \lincore is able to perform well in both easy and hard problems without tuning the parameters. For \lints and \linphe, they can achieve equally good performance as \lincore when the parameters are specially set for the problems. However, the parameters tuned for the easy problem under-explore in the hard problem and have almost linear regret. Similarly, the parameters tuned for the hard problem explore too much in the easy problem, and converge slowly.

We further tune the parameters $\alpha$ and $z$ of \lincore in the hard problem in \cref{fig:linear_tune} to see how it performs under different combinations of parameters. The results show that \lincore works well under a wide range of parameters and thus is easy to configure. For example, the area of $\alpha\in [0.4, 0.8]$ and $z\in [0.5,0.7]$ provides similarly competitive performance. When $\alpha$ and $z$ are too small, such as $\alpha = z = 0.2$, \lincore mainly exploits and explores too little to find the optimal arm. On the other hand, when $\alpha$ and $z$ are too large, such as $\alpha = 1.4$ and $z = 0.8$, it over-explores and suffers from high regret.
%explores aggressively \todob{over-explores and suffers from high regret?} to keep pulling random and sub-optimal arms and suffers from high regret. 
Moreover, it is worth noting that when setting $z$ to a large value, we have a large number of random pulls for initialization in order to have a high variance in the reward pool, which also leads to high regret in the early stage.

\begin{figure*}
    \centering
    \includegraphics[width=0.32\textwidth]{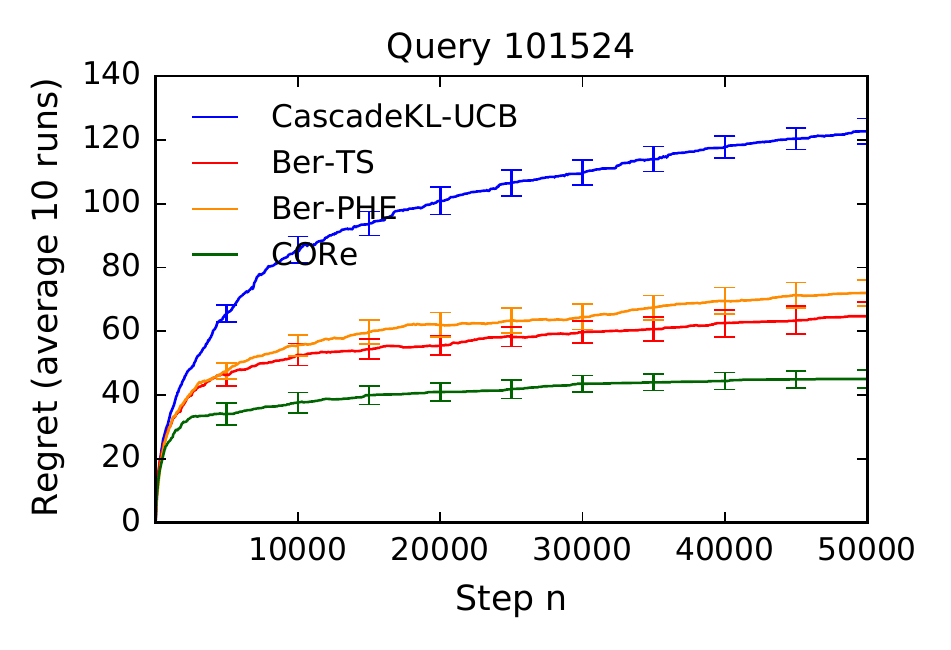}
    \includegraphics[width=0.32\textwidth]{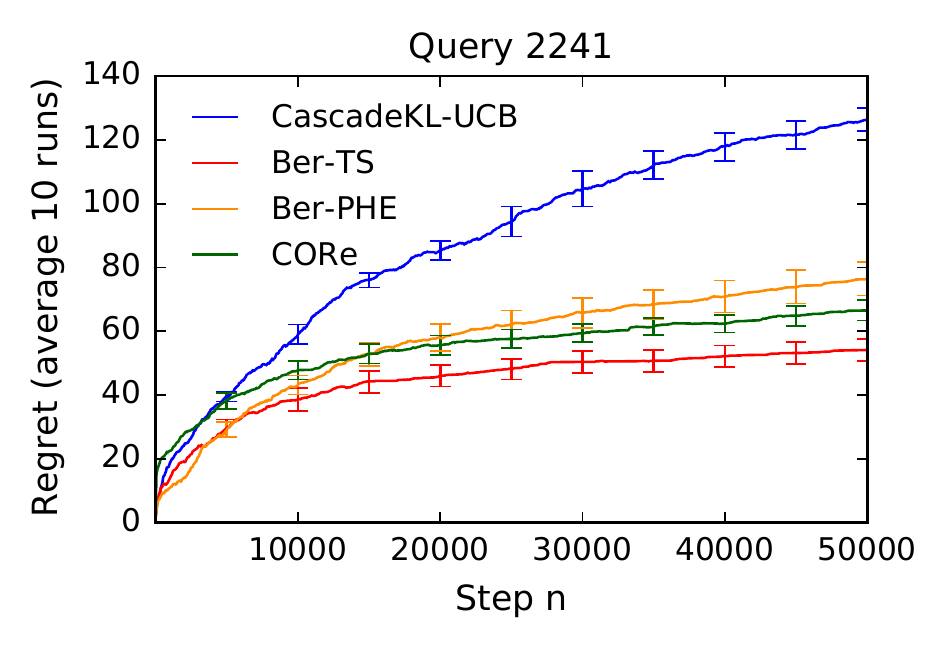}
    \includegraphics[width=0.32\textwidth]{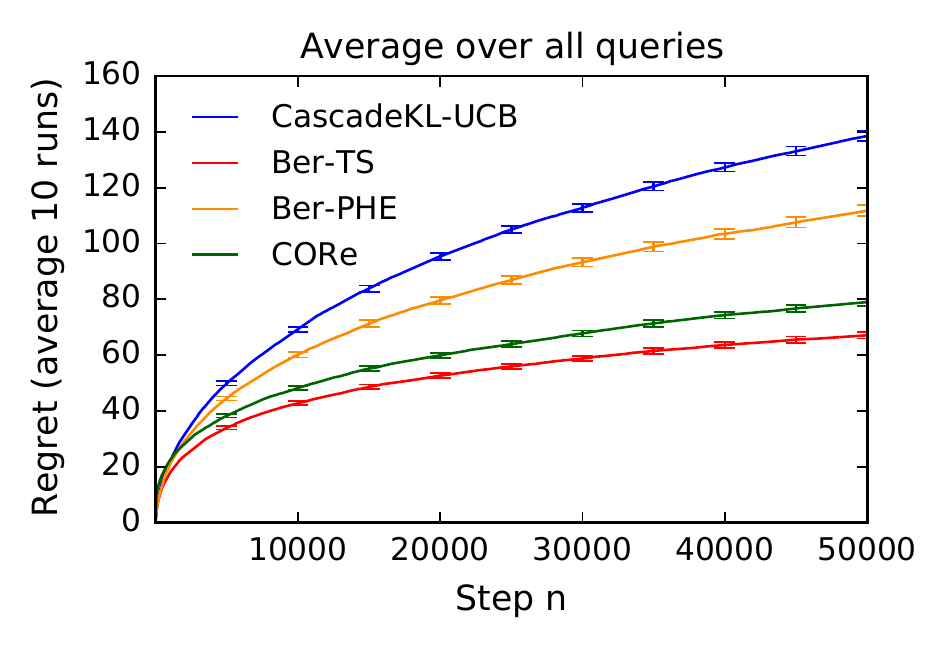}
    \caption{The cumulative regret of different algorithms in the learning to rank problem. The results are averaged over 10 runs per query. We sample two queries to demonstrate the performance in the first two figures, and display the results averaged over all queries in the third figure.}
    \label{fig:learning2rank}
\end{figure*}

\subsection{Online Learning to Rank}
\label{sec:l2r}

We finally evaluate \core in a real-world problem, online learning to rank \citep{liu2009l2r, radlinski2008diverse}. Online learning to rank is a sequential decision-making problem where the learning agent repeatedly recommends a list of items. In round $t$, the learning agent recommends a ranked list of $K$ items out of all $L \geq K$ items. The user clicks on the recommended items. The clicks are treated as bandit feedback. The performance of the agent is measured by the expected cumulative regret, which is the expected loss in clicks relatively to the optimal ranking.

We experiment with the \emph{Yandex} dataset and follow the experimental setup as in \citep{zoghi2017online, lattimore2018toprank}. In each query, the user is shown 10 documents and the search engine records clicks of the user. We use the $60$ most frequent queries from the dataset and learn their cascade models (CM) with PyClick \citep{Chuklin2015ClickMF}. The goal of the learning agent is to rerank $L=10$ most attractive items to maximize the expected number of clicks at the first $K=5$ positions. The application of bandit algorithms is similar as in the multi-armed bandit setting, despite that the agent will rank the items based on their mean reward estimates rather than selecting a single item. The corresponding cascade model learned under each query is used to generate clicks. We experiment with a horizon of $n=50,000$ rounds and the regret is averaged over $10$ runs. 

% \todob{What about TopRank? This algorithm is comparably general to our perturbation scheme and will be worse / comparable to CascadeKL-UCB.}

We compare \core to CascadeKL-UCB \citep{kveton2015cascading}, which is specifically designed for online learning to rank in the cascade model. We also evaluate Ber-\ts and Ber-\phe in this problem. They are applied in the same way as CascadeKL-UCB, with the UCB of each item replaced by its \ts or \phe mean reward estimate. TopRank \citep{lattimore2018toprank} is another algorithm for online learning to rank based on topological sort, but it is known to perform worse than CascadeKL-UCB and thus we do not include it. We still use the default parameters for \core ($\alpha=z=0.6$) and set $a=0.5$ in Ber-\phe. The results are presented in \cref{fig:learning2rank}, where we show the results under two specific queries in the first two figures and show the average performance over all queries in the third figure. Under the default parameters, \core already achieves competitive performance that consistently outperforms Ber-\phe and CascadeKL-UCB across queries, and is comparable to Ber-\ts. We also observed further improvement of \core if tuned, to $\alpha = z = 0.4$, which achieves almost the same performance as Ber-\ts when averaging over all queries. 
% \todob{Do we perform better? By how much? Say it. Otherwise the reviewer will think that the improvement is tiny and not worth reporting.} 
The promising results from this experiment demonstrate the wide applicability and robustness of \core in real-world structured problems, and its ability in solving a new problem without prior knowledge.

% This experiment demonstrates the wide applicability of \core in real-world structured problems. 

\section{Related Work}
\label{sec:related work}

The key to statistically-efficient exploration in stochastic bandits is to perturb the mean reward estimates of arms sufficiently. Algorithms based on upper confidence bounds (UCBs) \citep{auer2002finite,yasin2011improved} perturb the mean reward estimates by adding confidence intervals to them. The confidence intervals are constructed by theory. Although theoretically optimal, they are often conservative in practice, because they are designed for hardest problem instances. \ucbv \citep{audibert2009ucbv} is a variant of \ucb that adapts confidence intervals using an empirical estimate of the variance from observed rewards. This algorithm also tends to be conservative in practice, as we show in \cref{sec:multiarmed-bandit}.

Posterior sampling \citep{thompson1933on,agrawal2013furtherTS} introduces variance in mean reward estimates by sampling from posterior distributions. To be statistically efficient, proper variance needs to be specified in the posterior updates, which is often unknown in real-world problems. As we show in \cref{sec:tuning}, when the variance of the posterior in Gauss-\lints is mis-specified, the algorithm suffers from high regret, due to either under- or over-exploration. \core is closely related to posterior sampling in Gaussian bandits (\cref{sec:idea}). However, instead of relying on knowing the variance of reward distributions, it utilizes the randomness in the agent's observed rewards, to have its data-dependent exploration that adapts to problem hardness. 

Randomized exploration algorithms, such as \giro \citep{kveton2019garbage} and \phe \citep{kveton2019perturbedhistory}, add pseudo-rewards to the reward history and use the perturbed mean reward estimates for arm selection. The added pseudo-rewards add sufficient variance for exploration and lead to provably sublinear regret in multi-armed bandits. However, similarly to UCB designs and posterior sampling, the right amount of perturbation is needed to explore at a near-optimal rate. In contrast, instead of adding external noise from pseudo rewards, \core samples from the agent's past observed rewards to induce exploration. This provides sufficient variance when all reward distributions have comparable variance and we analyze \lincore theoretically in the setting of identical Gaussian noise. 
% \todom{and we analyze \core theoretically in the setting of identical Gaussian noise. why this sentence is necessary here}

The idea of efficient exploration with no prior knowledge on the arms' distribution has emerged in recent years. Non-parametric solutions have been proposed in the multi-armed bandit setting. The most representative works are non-parametric Thompson sampling (\npts) \citep{pmlr-v117-riou20a} and subsample-mean comparison (\ssmc) \citep{chan2019multiarmed}. Specifically, \npts proposes a generalization of the Bernoulli Thompson sampling to multinomial distributions, and a non-parametric adaption of this algorithm. \ssmc is inspired from the sub-sampling approaches \citep{baransi2014subsampling} and is asymptotically optimal for exponential families of distributions. We compare them with \core in the multi-armed bandit setting in \cref{sec:multiarmed-bandit}. However, it is unclear how to generalize \npts and \ssmc to linear bandits.  

\section{Conclusions}
\label{sec:conclusions}

We propose a new online algorithm, capitalizing on rewards (\core), that explores by utilizing the randomness of the agent's past observed rewards. In particular, \core sample rewards from a well designed reward pool from the agent's past observations to perturb the reward histories. The variance introduced by sampled rewards automatically adapts to the noise of the reward distributions. Thus \core can impose proper exploration in different problems without parameter tuning. We prove a $\tilde O(d\sqrt{n\log K})$ gap-free bound on the $n$-round regret of \core in a stochastic linear bandit. Our empirical evaluation shows that \core achieves competitive performance in various problems.   

\core is general enough to be applied to different structured problems, such as generalized linear bandits \citep{filippi2010parametric} or neural bandits \citep{zhou2020neural}. The randomization strategy remains the same for different problems. We analyze the regret of \core in a linear Gaussian bandit. Our analysis is under the assumption that the reward distributions of all arms have the same variance. An interesting future direction is a more general analysis of \core. 
% In addition, we believe that it is important to analyze \core in structured problems, similarly to other recently proposed randomized strategies \citep{kveton2019perturbedhistory}.

Finally, we also believe that \core can be further extended by other randomization designs, with the essential idea of capitalizing on the randomness in the agent's observed rewards and being fully data-dependent. For example, we can dynamically exchange rewards among arms with certain probability and keep the exchanged rewards in the arm's history along the $n$-round game. This can greatly improve the efficiency of sampling i.i.d.\ rewards from the reward pool in every single round. We have observed promising empirical performance of such algorithms and leave their more detailed study for future work.

\bibliography{references}

\begin{thebibliography}{38}
\providecommand{\natexlab}[1]{#1}
\providecommand{\url}[1]{\texttt{#1}}
\expandafter\ifx\csname urlstyle\endcsname\relax
  \providecommand{\doi}[1]{doi: #1}\else
  \providecommand{\doi}{doi: \begingroup \urlstyle{rm}\Url}\fi

\bibitem[Abbasi-Yadkori et~al.(2011)Abbasi-Yadkori, P\'{a}l, and
  Szepesv\'{a}ri]{yasin2011improved}
Yasin Abbasi-Yadkori, D\'{a}vid P\'{a}l, and Csaba Szepesv\'{a}ri.
\newblock Improved algorithms for linear stochastic bandits.
\newblock In J.~Shawe-Taylor, R.~S. Zemel, P.~L. Bartlett, F.~Pereira, and
  K.~Q. Weinberger, editors, \emph{Advances in Neural Information Processing
  Systems 24}, pages 2312--2320. Curran Associates, Inc., 2011.

\bibitem[Abeille and Lazaric(2016)]{abeille2016linear}
Marc Abeille and Alessandro Lazaric.
\newblock Linear thompson sampling revisited.
\newblock \emph{Electronic Journal of Statistics}, 11, 11 2016.
\newblock \doi{10.1214/17-EJS1341SI}.

\bibitem[Agrawal and Goyal(2012)]{agrawal2012thompson}
Shipra Agrawal and Navin Goyal.
\newblock Thompson sampling for contextual bandits with linear payoffs, 2012.

\bibitem[Agrawal and Goyal(2013)]{agrawal2013furtherTS}
Shipra Agrawal and Navin Goyal.
\newblock Further optimal regret bounds for thompson sampling.
\newblock In Carlos~M. Carvalho and Pradeep Ravikumar, editors,
  \emph{Proceedings of the Sixteenth International Conference on Artificial
  Intelligence and Statistics}, volume~31 of \emph{Proceedings of Machine
  Learning Research}, pages 99--107, Scottsdale, Arizona, USA, 29 Apr--01 May
  2013. PMLR.

\bibitem[Audibert et~al.(2009)Audibert, Munos, and
  Szepesv\'{a}ri]{audibert2009ucbv}
Jean-Yves Audibert, R\'{e}mi Munos, and Csaba Szepesv\'{a}ri.
\newblock Exploration-exploitation tradeoff using variance estimates in
  multi-armed bandits.
\newblock \emph{Theor. Comput. Sci.}, 410\penalty0 (19):\penalty0 1876–1902,
  April 2009.
\newblock ISSN 0304-3975.

\bibitem[Auer et~al.(2002)Auer, Cesa-Bianchi, and Fischer]{auer2002finite}
Peter Auer, Nicol\`{o} Cesa-Bianchi, and Paul Fischer.
\newblock Finite-time analysis of the multiarmed bandit problem.
\newblock \emph{Mach. Learn.}, 47\penalty0 (2–3):\penalty0 235–256, May
  2002.
\newblock ISSN 0885-6125.

\bibitem[Baransi et~al.(2014{\natexlab{a}})Baransi, Maillard, and
  Mannor]{baransi2014sub}
Akram Baransi, Odalric-Ambrym Maillard, and Shie Mannor.
\newblock Sub-sampling for multi-armed bandits.
\newblock In Toon Calders, Floriana Esposito, Eyke H{\"u}llermeier, and Rosa
  Meo, editors, \emph{Machine Learning and Knowledge Discovery in Databases},
  pages 115--131, Berlin, Heidelberg, 2014{\natexlab{a}}. Springer Berlin
  Heidelberg.
\newblock ISBN 978-3-662-44848-9.

\bibitem[Baransi et~al.(2014{\natexlab{b}})Baransi, Maillard, and
  Mannor]{baransi2014subsampling}
Akram Baransi, Odalric-Ambrym Maillard, and Shie Mannor.
\newblock Sub-sampling for multi-armed bandits.
\newblock In Toon Calders, Floriana Esposito, Eyke H{\"u}llermeier, and Rosa
  Meo, editors, \emph{Machine Learning and Knowledge Discovery in Databases},
  pages 115--131, Berlin, Heidelberg, 2014{\natexlab{b}}. Springer Berlin
  Heidelberg.

\bibitem[Chan(2019)]{chan2019multiarmed}
Hock~Peng Chan.
\newblock The multi-armed bandit problem: An efficient non-parametric solution,
  2019.

\bibitem[Chapelle and Li(2011)]{chapelle2011empirical}
Olivier Chapelle and Lihong Li.
\newblock An empirical evaluation of thompson sampling.
\newblock In J.~Shawe-Taylor, R.~S. Zemel, P.~L. Bartlett, F.~Pereira, and
  K.~Q. Weinberger, editors, \emph{Advances in Neural Information Processing
  Systems 24}, pages 2249--2257. Curran Associates, Inc., 2011.

\bibitem[Chuklin et~al.(2015)Chuklin, Markov, and Rijke]{Chuklin2015ClickMF}
A.~Chuklin, I.~Markov, and M.~Rijke.
\newblock Click models for web search.
\newblock In \emph{Click Models for Web Search}, 2015.

\bibitem[Dasgupta and Gupta(2003)]{dasgupta2003elementary}
Sanjoy Dasgupta and Anupam Gupta.
\newblock An elementary proof of a theorem of johnson and lindenstrauss.
\newblock \emph{Random Struct. Algorithms}, 22\penalty0 (1):\penalty0 60–65,
  January 2003.
\newblock ISSN 1042-9832.

\bibitem[Eckles and Kaptein(2014)]{eckles2014Thompson}
D.~Eckles and M.~Kaptein.
\newblock Thompson sampling with the online bootstrap.
\newblock \emph{ArXiv}, abs/1410.4009, 2014.

\bibitem[Filippi et~al.(2010)Filippi, Cappe, Garivier, and
  Szepesv\'{a}ri]{filippi2010parametric}
Sarah Filippi, Olivier Cappe, Aur\'{e}lien Garivier, and Csaba Szepesv\'{a}ri.
\newblock Parametric bandits: The generalized linear case.
\newblock In J.~D. Lafferty, C.~K.~I. Williams, J.~Shawe-Taylor, R.~S. Zemel,
  and A.~Culotta, editors, \emph{Advances in Neural Information Processing
  Systems 23}, pages 586--594. Curran Associates, Inc., 2010.

\bibitem[Gopalan et~al.(2013)Gopalan, Mannor, and Mansour]{gopalan2013thompson}
Aditya Gopalan, Shie Mannor, and Yishay Mansour.
\newblock Thompson sampling for complex bandit problems, 2013.

\bibitem[Kim and Tewari(2019)]{baekjin2019on}
Baekjin Kim and Ambuj Tewari.
\newblock On the optimality of perturbations in stochastic and adversarial
  multi-armed bandit problems.
\newblock In \emph{Advances in Neural Information Processing Systems 32}, pages
  2695--2704. 2019.

\bibitem[Kveton et~al.(2015)Kveton, Szepesvari, Wen, and
  Ashkan]{kveton2015cascading}
Branislav Kveton, Csaba Szepesvari, Zheng Wen, and Azin Ashkan.
\newblock Cascading bandits: Learning to rank in the cascade model, 2015.

\bibitem[Kveton et~al.(2019{\natexlab{a}})Kveton, Szepesvari, Ghavamzadeh, and
  Boutilier]{kveton2019perturbedhistory}
Branislav Kveton, Csaba Szepesvari, Mohammad Ghavamzadeh, and Craig Boutilier.
\newblock Perturbed-history exploration in stochastic multi-armed bandits,
  2019{\natexlab{a}}.

\bibitem[Kveton et~al.(2019{\natexlab{b}})Kveton, Szepesvari, Vaswani, Wen,
  Lattimore, and Ghavamzadeh]{kveton2019garbage}
Branislav Kveton, Csaba Szepesvari, Sharan Vaswani, Zheng Wen, Tor Lattimore,
  and Mohammad Ghavamzadeh.
\newblock Garbage in, reward out: Bootstrapping exploration in multi-armed
  bandits.
\newblock volume~97 of \emph{Proceedings of Machine Learning Research}, pages
  3601--3610, Long Beach, California, USA, 09--15 Jun 2019{\natexlab{b}}. PMLR.

\bibitem[Kveton et~al.(2020)Kveton, Szepesv{\'{a}}ri, Ghavamzadeh, and
  Boutilier]{kveton2020linphe}
Branislav Kveton, Csaba Szepesv{\'{a}}ri, Mohammad Ghavamzadeh, and Craig
  Boutilier.
\newblock Perturbed-history exploration in stochastic linear bandits.
\newblock volume 115 of \emph{Proceedings of Machine Learning Research}, pages
  530--540, Tel Aviv, Israel, 22--25 Jul 2020. PMLR.

\bibitem[Lai and Robbins(1985)]{Lai1985asymptotically}
T.L Lai and Herbert Robbins.
\newblock Asymptotically efficient adaptive allocation rules.
\newblock \emph{Adv. Appl. Math.}, 6\penalty0 (1):\penalty0 4–22, March 1985.
\newblock ISSN 0196-8858.

\bibitem[Lattimore and Szepesv{\'a}ri(2020)]{lattimore2020bandit}
T.~Lattimore and C.~Szepesv{\'a}ri.
\newblock \emph{Bandit Algorithms}.
\newblock Cambridge University Press, 2020.
\newblock ISBN 9781108486828.

\bibitem[Lattimore et~al.(2018)Lattimore, Kveton, Li, and
  Szepesv\'{a}ri]{lattimore2018toprank}
Tor Lattimore, Branislav Kveton, Shuai Li, and Csaba Szepesv\'{a}ri.
\newblock Toprank: A practical algorithm for online stochastic ranking.
\newblock In \emph{Proceedings of the 32nd International Conference on Neural
  Information Processing Systems}, NIPS'18, page 3949–3958, Red Hook, NY,
  USA, 2018. Curran Associates Inc.

\bibitem[Li et~al.(2017)Li, Lu, and Zhou]{provably2017li}
Lihong Li, Yu~Lu, and Dengyong Zhou.
\newblock Provably optimal algorithms for generalized linear contextual
  bandits.
\newblock In \emph{Proceedings of the 34th International Conference on Machine
  Learning - Volume 70}, ICML'17, page 2071–2080. JMLR.org, 2017.

\bibitem[Liu(2009)]{liu2009l2r}
Tie-Yan Liu.
\newblock Learning to rank for information retrieval.
\newblock \emph{Found. Trends Inf. Retr.}, 3\penalty0 (3):\penalty0 225–331,
  March 2009.
\newblock ISSN 1554-0669.

\bibitem[McCullagh(1984)]{maccullagh1984generalized}
Peter McCullagh.
\newblock Generalized linear models.
\newblock \emph{European Journal of Operational Research}, 16\penalty0
  (3):\penalty0 285--292, 1984.

\bibitem[Osband and Roy(2015)]{osband2015bootstrapped}
Ian Osband and Benjamin~Van Roy.
\newblock Bootstrapped thompson sampling and deep exploration, 2015.

\bibitem[Radlinski et~al.(2008)Radlinski, Kleinberg, and
  Joachims]{radlinski2008diverse}
Filip Radlinski, Robert Kleinberg, and Thorsten Joachims.
\newblock Learning diverse rankings with multi-armed bandits.
\newblock ICML '08, page 784–791, New York, NY, USA, 2008. Association for
  Computing Machinery.

\bibitem[Riou and Honda(2020)]{pmlr-v117-riou20a}
Charles Riou and Junya Honda.
\newblock Bandit algorithms based on thompson sampling for bounded reward
  distributions.
\newblock In Aryeh Kontorovich and Gergely Neu, editors, \emph{Proceedings of
  the 31st International Conference on Algorithmic Learning Theory},
  Proceedings of Machine Learning Research, San Diego, California, USA, 2020.
  PMLR.

\bibitem[Riquelme et~al.(2018)Riquelme, Tucker, and Snoek]{riquelme2018deep}
Carlos Riquelme, George Tucker, and Jasper Snoek.
\newblock Deep bayesian bandits showdown: An empirical comparison of bayesian
  deep networks for thompson sampling.
\newblock In \emph{International Conference on Learning Representations}, 2018.

\bibitem[Rusmevichientong and Tsitsiklis(2008)]{paat2008linearly}
Paat Rusmevichientong and John~N. Tsitsiklis.
\newblock Linearly parameterized bandits, 2008.

\bibitem[Tang et~al.(2015)Tang, Jiang, Li, Zeng, and Li]{tang2015personalized}
Liang Tang, Yexi Jiang, Lei Li, Chunqiu Zeng, and Tao Li.
\newblock Personalized recommendation via parameter-free contextual bandits.
\newblock In \emph{Proceedings of the 38th International ACM SIGIR Conference
  on Research and Development in Information Retrieval}, SIGIR '15, page
  323–332, New York, NY, USA, 2015. Association for Computing Machinery.
\newblock ISBN 9781450336215.

\bibitem[Thompson(1933)]{thompson1933on}
William~R Thompson.
\newblock {On the Likelihood that One Unknown Probability Exceeds Another in
  View of the Evidence of Two Samples}.
\newblock \emph{Biometrika}, 25\penalty0 (3-4):\penalty0 285--294, 12 1933.

\bibitem[Vaswani et~al.(2018)Vaswani, Kveton, Wen, Rao, Schmidt, and
  Abbasi-Yadkori]{vaswani2018new}
Sharan Vaswani, Branislav Kveton, Zheng Wen, Anup Rao, Mark Schmidt, and Yasin
  Abbasi-Yadkori.
\newblock New insights into bootstrapping for bandits, 2018.

\bibitem[Vaswani et~al.(2020)Vaswani, Mehrabian, Durand, and
  Kveton]{vaswani2020old}
Sharan Vaswani, Abbas Mehrabian, Audrey Durand, and Branislav Kveton.
\newblock Old dog learns new tricks: Randomized ucb for bandit problems.
\newblock volume 108 of \emph{Proceedings of Machine Learning Research}, pages
  1988--1998, Online, 26--28 Aug 2020. PMLR.

\bibitem[Zhang et~al.(2016)Zhang, Yang, Jin, Xiao, and Zhou]{zhang2016online}
Lijun Zhang, Tianbao Yang, Rong Jin, Yichi Xiao, and Zhi-Hua Zhou.
\newblock Online stochastic linear optimization under one-bit feedback.
\newblock In \emph{Proceedings of the 33rd International Conference on
  International Conference on Machine Learning - Volume 48}, ICML'16, page
  392–401. JMLR.org, 2016.

\bibitem[Zhou et~al.(2020)Zhou, Li, and Gu]{zhou2020neural}
Dongruo Zhou, Lihong Li, and Quanquan Gu.
\newblock Neural contextual bandits with ucb-based exploration, 2020.

\bibitem[Zoghi et~al.(2017)Zoghi, Tunys, Ghavamzadeh, Kveton, Szepesvari, and
  Wen]{zoghi2017online}
Masrour Zoghi, Tomas Tunys, Mohammad Ghavamzadeh, Branislav Kveton, Csaba
  Szepesvari, and Zheng Wen.
\newblock Online learning to rank in stochastic click models.
\newblock In \emph{Proceedings of the 34th International Conference on Machine
  Learning - Volume 70}, ICML'17, page 4199–4208. JMLR.org, 2017.

\end{thebibliography}

\clearpage
\onecolumn
\appendix

\section{Proofs}
\label{sec:proofs}

The analysis is organized as follows. In \cref{sec:background}, we provide necessary technical background. In \cref{sec:regret bound}, we state and prove our regret bound. In \cref{sec:regret lemmas}, we present and prove key lemmas used in the regret bound. In \cref{sec:reward pool lemmas}, we prove two key lemmas that characterize sufficient exploratory properties of the reward pool.

\subsection{Background}
\label{sec:background}

For an event $E$, $\mathbbm{1}\{E\}=1$ if $E$ occurs and $\mathbbm{1}\{E\}=0$ otherwise. A random variable $X$ is $\rho^2$-sub-Gaussian if $\E{\exp(\lambda(X-\E X))}\leq \exp(\lambda^2\rho^2/2)$ for any $\lambda>0$. Let $L=\max_{i\in[K]}||x_i||_2$ and $L_\ast = ||\theta_*||_2$ be the maximum $l_2$ norm of feature vectors and the $l_2$ norm of the parameter vector, respectively.

By definition, $Y_\ell - X_\ell\T \theta_* \sim \cN(0, \sigma^2)$. We denote the $\ell$-th drawn reward from the reward pool by $Z_\ell$. We assume that $\abs{Z_\ell} \leq a$ almost surely and $\var{Z_\ell} \geq \eta^2$. We instantiate $a$ and $\eta$ in \cref{sec:regret bound}. We denote by
\begin{align}
  \hat{\theta}_t
  = G_t^{-1} \sum_{\ell = 1}^{t - 1} X_\ell Y_\ell
  \label{eq:theta hat}
\end{align}
the parameter vector estimated from rewards $Y_\ell$ and by
\begin{align}
  \tilde{\theta}_t
  = G_t^{-1} \sum_{\ell = 1}^{t - 1} X_\ell (Y_\ell + Z_\ell)
  \label{eq:theta tilde}
\end{align}
the parameter vector estimated from perturbed rewards $Y_\ell + Z_\ell$.

Let $\cF_t = \sigma(I_1, \dots, I_t, Y_{I_1, 1}, \dots, Y_{I_t, t})$ be the $\sigma$-algebra generated by the pulled arms and their rewards by the end of round $t \in [n] \cup \set{0}$. We define $\cF_0 = \set{\emptyset, \Omega}$, where $\Omega$ is the sample space of the probability space that holds all random variables. We denote by $\probt{\cdot} = \condprob{\cdot}{\cF_{t - 1}}$ and $\Et{\cdot} = \condE{\cdot}{\cF_{t - 1}}$ the conditional probability and expectation operators, respectively, given the past at the beginning of round $t$. Let $\normw{x}{M} = \sqrt{x\T M x}$. Let
\begin{align}
  E_{1, t}
  = \set{\forall i \in [K]:
  \abs{x_i\T \hat{\theta}_t - x_i\T \theta_\ast} \leq
  c_1 \normw{x_i}{G_t^{-1}}}
  \label{eq:theta hat is close}
\end{align}
be the event that $\hat{\theta}_t$ is \say{close} to $\theta_\ast$ in round $t$, where $\hat{\theta}_t$ is defined in \eqref{eq:theta hat} and $c_1 > 0$ is tuned such that $\bar{E}_{1, t}$, the complement of $E_{1, t}$, is unlikely. Let $E_1 = \bigcap_{t = d + 1}^n E_{1, t}$ and $\bar{E}_1$ be its complement. Let
\begin{align}
  E_{2, t}
  = \set{\forall i \in [K]:
  \abs{x_i\T \tilde{\theta}_t - x_i\T \hat{\theta}_t}
  \leq c_2 \normw{x_i}{G_t^{-1}}}
  \label{eq:theta tilde is close}
\end{align}
be the event that $\tilde{\theta}_t$ is \say{close} to $\hat{\theta}_t$ in round $t$, where $\tilde{\theta}_t$ is defined in \eqref{eq:theta tilde} and $c_2 > 0$ is tuned such that $\bar{E}_{2, t}$, the complement of $E_{2, t}$, is unlikely given any past.

Our bound involves three probability constants. The first constant, $p_1$, is an upper bound on the probability of event $\bar{E}_1$, that is $p_1 \geq \prob{\bar{E}_1}$. The second constant, $p_2$, is an upper bound on the probability of event $\bar{E}_{2, t}$ given any past,
\begin{align}
  \probt{\bar{E}_{2, t}}
  \leq p_2\,.
  \label{eq:p2}
\end{align}
The last constant, $p_3$, is a lower bound on the probability that the optimal arm $1$ is optimistic given any past,
\begin{align}
  \probt{x_1\T \tilde{\theta}_t - x_1\T \hat{\theta}_t > c_1 \normw{x_1}{G_t^{-1}}}
  \geq p_3\,.
  \label{eq:p3}
\end{align}
Using the above notation, we restate the general regret bound for linear bandits of \citet{kveton2020linphe}.

\begin{theorem}
\label{thm:lin upper bound} Let $c_1, c_2 \geq 1$. Let $A$ be any algorithm that pulls arm $I_t = \argmax_{i \in [K]} x_i\T \tilde{\theta}_t$ in round $t$, where $\tilde{\theta}_t$ is estimated from past data. Let the mean rewards be in $[0, 1]$; $p_1$, $p_2$, and $p_3$ be defined as above; and $p_3 > p_2$. Then the expected $n$-round regret of $A$ is bounded as
\begin{align*}
  R(n)
  \leq (c_1 + c_2) \left(1 + \frac{2}{p_3 - p_2}\right) \sqrt{c_3 n} + (p_1 + p_2) n + d\,,
\end{align*}
where $c_3 =  2 d \log(1 + n L^2 / (d \lambda))$.
\end{theorem}

\subsection{Regret Bound}
\label{sec:regret bound}

We prove our regret bound by instantiating \cref{thm:lin upper bound}. In summary, we have that
\begin{align*}
  c_1
  = \tilde{O}(\sqrt{d})\,, \quad
  c_2
  = \tilde{O}(\sqrt{d \log K})\,, \quad
  p_1
  = O(1 / n)\,, \quad
  p_2
  = O(1 / n)\,, \quad
  \frac{1}{p_3 - p_2}
  = \tilde{O}(1)\,.
\end{align*}
Therefore, \cref{thm:lin upper bound} yields the following regret bound.

\lincoreregret*
\begin{proof}
In the first $\max\big\{d,\frac{4\log n}{z-1-\log z}+1\big\}$ rounds, we bound the regret trivially. After these initial rounds, the bounds in \cref{lemma:init-variance,lemma:bound_reward} hold jointly with probability at least $1 - 2 / n$ over all remaining rounds. Since the bounds fail with probability at most $2 / n$, the expected $n$-round regret due to the failures is at most $2$. So, the regret due to the initialization and the bound failures is $\tilde{O}(d)$, and is subsumed by the $\tilde{O}(d \sqrt{n \log K})$ term in the regret bound.

Now we focus on instantiating \cref{thm:lin upper bound}. First, we set $c_1$ as in \cref{lem:theta hat concentration}. For $\delta = 1 / n$, we have
\begin{align*}
  c_1
  = \sigma \sqrt{d \log(n + n^2 L^2 / (d \lambda))} +
  \lambda^\frac{1}{2} L_\ast
\end{align*}
and $p_1 = 1 / n$. Then we set
\begin{align*}
  c_2
  = \sqrt{2 a^2 \log(K n^4)}\,.
\end{align*}
By \cref{lem:theta tilde concentration} for $c = c_2$, we have that $p_2 = 1 / n^4$. Finally, we set $p_3$ using \cref{lem:theta tilde anti-concentration}. In particular,
\begin{align*}
  p_3
  = \frac{1}{16 \log n}
  \left[\frac{\eta^2}{a^2} \left(1 - \frac{\lambda}{\lambda_{\min}(G_{d + 1})}\right) -
  \frac{c^2}{a^2} - \frac{2}{n^3}\right]
\end{align*}
and we work out a nicer algebraic form in the rest of the proof. First, we set $\lambda = \lambda_{\min}(G_{d + 1}) / 4$. Note that this is well defined since $G_{d + 1}$ is deterministic. For this setting,
\begin{align*}
  p_3
  = \frac{1}{16 \log n}
  \left[\frac{3}{4} \frac{\eta^2}{a^2} - \frac{c^2}{a^2} - \frac{2}{n^3}\right]\,.
\end{align*}
Now we set $\eta$ and $a$ as in \lincore (\cref{lemma:init-variance,lemma:bound_reward}),
\begin{align*}
  \eta^2
  = \frac{z}{2} \alpha^2 \sigma^2\,, \quad
  a
  = 4 \sqrt{\alpha^2 \sigma^2 \log n} + \alpha\,.
\end{align*}
Now we set $c = c_1$ and $\eta^2 = 2 c_1^2$, which means that $\alpha^2 = \frac{4}{z} \sigma^{-2} c_1^2$. Then
\begin{align*}
  p_3
  = \frac{1}{32 \log n}
  \left[\frac{c_1^2}{a^2} - \frac{4}{n^3}\right]\,.
\end{align*}
Under the assumption that $4 \sqrt{\sigma^2 \log n} \geq 1$, we have $a^2 \leq 64 \alpha^2 \sigma^2 \log n$. Moreover, for the above setting of $\alpha^2$, we have that $a^2 \leq 256 z^{-1} c_1^2 \log n$. Thus
\begin{align*}
  p_3
  \geq \frac{1}{32 \log n}
  \left[\frac{z}{256 \log n} - \frac{4}{n^3}\right]\,.
\end{align*}
Finally, since $p_2 = 1 / n^4$, we have that
\begin{align*}
  p_3 - p_2
  \geq \frac{1}{32 \log n}
  \left[\frac{z}{256 \log n} - \frac{4}{n^3} - \frac{32 \log n}{n^4}\right]
  \geq \frac{1}{32 \log n} \left[\frac{z}{256 \log n} - \frac{36}{n^3}\right]\,.
\end{align*}
Now note that for $z \geq 1 / 2$, we have $36 / n^3 \leq z / (512 \log n)$ for $n \geq 24$. This means that $1 / (p_3 - p2) = \tilde{O}(1)$ for $n \geq 24$. This concludes the proof.
\end{proof}

\subsection{Regret Lemmas}
\label{sec:regret lemmas}

A standard concentration lemma is below.

\begin{lemma}[Least-squares concentration]
\label{lem:theta hat concentration} For any $\lambda > 0$, $\delta > 0$, and
\begin{align*}
  c_1
  = \sigma \sqrt{d \log \left(\frac{1 + n L^2 / (d \lambda)}{\delta}\right)} +
  \lambda^\frac{1}{2} L_\ast\,,
\end{align*}
event $E_1$ occurs with probability at least $1 - \delta$.
\end{lemma}
\begin{proof}
By the Cauchy-Schwarz inequality,
\begin{align*}
  x_i\T \hat{\theta}_t - x_i\T \theta_\ast
  = x_i\T G_t^{- \frac{1}{2}} G_t^\frac{1}{2} (\hat{\theta}_t - \theta_\ast)
  \leq \|\hat{\theta}_t - \theta_\ast\|_{G_t} \normw{x_i}{G_t^{-1}}\,.
\end{align*}
Now note that the least-squares estimate $\hat{\theta}_t$ is computed from $\sigma^2$-sub-Gaussian rewards. As a result, by Theorem 2 of \citet{abbasi-yadkori11improved} for $R = \sigma$, $\|\hat{\theta}_t - \theta_\ast\|_{G_t} \leq c_1$ holds jointly in all rounds $t \leq n$ with probability of at least $1 - \delta$. This completes the proof.
\end{proof}

The concentration lemma for perturbation noise is below.

\begin{lemma}
\label{lem:theta tilde concentration} For any $t > d$, $c > 0$, and vector $x \in \realset^d$, we have
\begin{align*}
  \probt{\abs{x\T \tilde{\theta}_t - x\T \hat{\theta}_t} \geq c \normw{x}{G_t^{-1}}}
  \leq 2 \exp\left[- \frac{c^2}{2 a^2}\right]\,.
\end{align*}
\end{lemma}
\begin{proof}
Let
\begin{align*}
  U
  & = x\T \tilde{\theta}_t
  = \sum_{\ell = 1}^{t - 1} x\T G_t^{-1} X_\ell (Y_\ell + Z_\ell)\,, \\
  \bar{U}
  & = x\T \hat{\theta}_t
  = \sum_{\ell = 1}^{t - 1} x\T G_t^{-1} X_\ell Y_\ell\,,
\end{align*}
and $D = U - \bar{U}$. Then by Hoeffding's inequality,
\begin{align*}
  \probt{\abs{x\T \tilde{\theta}_t - x\T \hat{\theta}_t} \geq c \normw{x}{G_t^{-1}}}
  = \probt{\abs{D} \geq c \normw{x}{G_t^{-1}}}
  \leq 2 \exp\left[- \frac{c^2 \normw{x}{G_t^{-1}}^2}
  {2 a^2 \sum_{\ell = 1}^{t - 1} x\T G_t^{-1} X_\ell X_\ell\T G_t^{-1} x}\right]\,.
\end{align*}
This step of the proof relies on the fact that new $Z_\ell \in [- a, a]$ are generated in each round $t$. Also note that
\begin{align}
  \sum_{\ell = 1}^{t - 1} x\T G_t^{-1} X_\ell X_\ell\T G_t^{-1} x
  \leq x\T G_t^{-1} \left(\sum_{\ell = 1}^{t - 1}
  X_\ell X_\ell\T + \lambda I_d\right) G_t^{-1} x
  = \normw{x}{G_t^{-1}}^2\,.
  \label{eq:cancel G inverse}
\end{align}
Our claim follows from chaining all above inequalities.
\end{proof}

The key anti-concentration lemma for perturbation noise is below.

\begin{lemma}
\label{lem:theta tilde anti-concentration} For any round $t > d$, constant $c$ such that $8 a^2 \log n > c^2 > 0$, and vector $x \in \realset^d$ such that $x \neq \mathbf{0}$, we have
\begin{align*}
  \probt{x\T \tilde{\theta}_t - x\T \hat{\theta}_t > c \normw{x}{G_t^{-1}}}
  \geq \frac{1}{16 \log n}
  \left[\frac{\eta^2}{a^2} \left(1 - \frac{\lambda}{\lambda_{\min}(G_{d + 1})}\right) -
  \frac{c^2}{a^2} - \frac{2}{n^3}\right]\,.
\end{align*}
\end{lemma}
\begin{proof}
Let $U$, $\bar{U}$, and $D$ be defined as in the proof of \cref{lem:theta tilde concentration}. Then $x\T \tilde{\theta}_t - x\T \hat{\theta}_t = D$. We also define events
\begin{align*}
  F_1
  = \set{\abs{D} \leq c \normw{x}{G_t^{-1}}}\,, \quad
  F_2
  = \set{\abs{D} \leq \sqrt{8 a^2 \log n} \normw{x}{G_t^{-1}}}\,.
\end{align*}
Since $8 a^2 \log n > c^2$, $F_1 \subset F_2$. Then
\begin{align*}
  \condvar{U}{\cF_{t - 1}}
  = \Et{D^2 \I{F_1}} +
  \Et{D^2 \I{\bar{F}_1, F_2}} +
  \Et{D^2 \I{\bar{F}_2}}\,.
\end{align*}
Now we bound each term on the right-hand side of the above equality from above. From the definition of event $F_1$, term $1$ is bounded as
\begin{align*}
  \Et{D^2 \I{F_1}}
  \leq c^2 \normw{x}{G_t^{-1}}^2\,.
\end{align*}
By the definition of $F_1$ and $F_2$, term $2$ is bounded as
\begin{align*}
  \Et{D^2 \I{\bar{F}_1, F_2}}
  \leq (8 a^2 \normw{x}{G_t^{-1}}^2 \log n) \, \probt{\bar{F}_1, F_2 \text{ occur}}
  \leq (8 a^2 \normw{x}{G_t^{-1}}^2 \log n) \, \probt{\abs{D} > c \normw{x}{G_t^{-1}}}\,.
\end{align*}
Now we bound term $3$. First, note that
\begin{align*}
  \abs{D}
  \leq a \sum_{\ell = 1}^{t - 1} \abs{x\T G_t^{-1} X_\ell}
  \leq a \sqrt{n} \sqrt{\sum_{\ell = 1}^{t - 1}
  x\T G_t^{-1} X_\ell X_\ell\T G_t^{-1} x}
  \leq a \sqrt{n} \normw{x}{G_t^{-1}}\,,
\end{align*}
where the last step follows from \eqref{eq:cancel G inverse}. Then, by the definition of event $F_2$ and \cref{lem:theta tilde concentration} for $c = \sqrt{8 a^2 \log n}$,
\begin{align*}
  \Et{D^2 \I{\bar{F}_2}} 
  \leq a^2 n \normw{x}{G_t^{-1}}^2 \probt{\bar{F}_2}
  \leq \frac{2 a^2 \normw{x}{G_t^{-1}}^2}{n^3}\,.
\end{align*}
Finally, by the definition of $U$,
\begin{align*}
  \condvar{U}{\cF_{t - 1}}
  \geq \eta^2 \sum_{\ell = 1}^{t - 1} x\T G_t^{-1} X_\ell X_\ell\T G_t^{-1} x
  = \eta^2 (\normw{x}{G_t^{-1}}^2 - \lambda x\T G_t^{-2} x)\,.
\end{align*}
We bound the last term from below as follows. For any positive semi-definite matrix $M \in \realset^{d \times d}$,
\begin{align*}
  x\T M^2 x
  & = \lambda_{\max}^2(M) \, x\T \left(\lambda_{\max}^{-2}(M) M^2\right) x
  \leq \lambda_{\max}^2(M) \, x\T \left(\lambda_{\max}^{-1}(M) M\right) x \\
  & = \lambda_{\max}(M) \normw{x}{M}^2\,,
\end{align*}
where the inequality follows from the fact that all eigenvalues of $\lambda_{\max}^{-2}(M) M^2$ are in $[0, 1]$. We apply this upper bound for $M = G_t^{-1}$ and get that
\begin{align*}
  \condvar{U}{\cF_{t - 1}}
  \geq \eta^2 \left(1 - \frac{\lambda}{\lambda_{\min}(G_t)}\right)
  \normw{x}{G_t^{-1}}^2
  \geq \eta^2 \left(1 - \frac{\lambda}{\lambda_{\min}(G_{d + 1})}\right)
  \normw{x}{G_t^{-1}}^2\,,
\end{align*}
where the last inequality is by $\lambda_{\min}(G_t) \geq \lambda_{\min}(G_{d + 1})$ and holds for any $t > d$.

Now we combine all above inequalities and get
\begin{align*}
  \left[\eta^2 \left(1 - \frac{\lambda}{\lambda_{\min}(G_{d + 1})}\right) -
  c^2 - \frac{2 a^2}{n^3}\right] \normw{x}{G_t^{-1}}^2
  \leq (8 a^2 \normw{x}{G_t^{-1}}^2 \log n) \,
  \probt{\abs{D} > c \normw{x}{G_t^{-1}}}\,.
\end{align*}
Since $2 a \log n > 0$ and $\normw{x}{G_t^{-1}} > 0$, the above inequality can be simplified as
\begin{align*}
  \probt{\abs{D} > c \normw{x}{G_t^{-1}}}
  \geq \frac{1}{8 \log n}
  \left[\frac{\eta^2}{a^2} \left(1 - \frac{\lambda}{\lambda_{\min}(G_{d + 1})}\right) -
  \frac{c^2}{a^2} - \frac{2}{n^3}\right]\,.
\end{align*}
Finally, we note that the distribution of $D$ is symmetric. Thus for any $\eps > 0$, $\probt{\abs{D} > \eps} = 2 \probt{D > \eps}$. This completes the proof.
\end{proof}

\subsection{Reward Pool Lemmas}
\label{sec:reward pool lemmas}

\initialvariance*
% \begin{lemma}
% \label{lemma:init-variance}
% For any $n\geq 2$, $K\geq 2$ and $z\in(0,1)$, $\sigma^2(\mathcal R_t)\geq \frac{z}{2}\sigma^2$ with probability of at least $1-\frac{1}{n}$, for all rounds $t$ that $t-1 \geq \max\big\{K,\frac{4\log n}{z-1-\log z}+1\big\}$. 
% \end{lemma}
\begin{proof}
Given a random variable $X\sim\mathcal N(\mu,\sigma^2)$, and a sample of it $(x_i)_{i=1}^q$ of size $q$. The sample variance is 
\begin{equation*}
    S^2 = \frac{1}{q-1}\sum_{i=1}^q (x_i-\bar x)^2,
\end{equation*}
where $\bar x = \frac{1}{q}\sum_{i=1}^q x_i$ is the mean of these $q$ observations. $S^2$ follows a scaled chi-squared distribution with ($q-1$) degrees of freedom,
\begin{equation}
    S^2\sim\frac{\sigma^2}{q-1}\chi^2_{q-1}. 
\end{equation}
The mean of $\chi^2_{q-1}$ is $q-1$. Let $0<z<1$ and $F(z(q-1);q-1)=\mathbb P(S^2<z\sigma^2)$ be the cumulative distribution function (CDF) of $\chi^2_{q-1}$. Based on Chernoff bounds on the lower tail of the CDF \citep{dasgupta2003elementary}, we have
\begin{equation*}
    F(z(q-1);q-1)\leq(ze^{1-z})^{(q-1)/2}. 
\end{equation*}
Let $(ze^{1-z})^{(q-1)/2} \leq 1/n^2$ and solve for $q$, we have $P(S^2<z\sigma^2)\leq 1/n^2$ when $q\geq \frac{4\log n}{z-1-\log z}+1$. However, the empirical variance of this sample is 
\begin{equation}
    \hat\sigma^2 = \frac{1}{q}\sum_{i=1}^q (x_i-\bar x)^2 = \frac{q-1}{q}S^2.
\end{equation}
Thus, $\hat\sigma^2 \geq \frac{S^2}{2} \geq \frac{z}{2}\sigma^2$ with probability of at least $1-1/n^2$, as long as $q\geq 2$.  

Besides, note that the sampled values in the reward pool $\mathcal R_t$ are from a mixture of Gaussian distributions with different means, rather than a single Gaussian. Denote the sample variance of $q$ sampled values from a mixture of Gaussian distributions with the same variance $\sigma^2$ by $\tilde S^2$. \cref{lem:general chi-square upper bound} claims that $\prob{S^2\leq \varepsilon^2} \geq \prob{\tilde S^2\leq \varepsilon^2}$. Therefore, $\sigma^2(\mathcal R_t) \geq \frac{z}{2}\sigma^2$ with probability of at least $1/n^2$ for any $t > \frac{4\log n}{z-1-\log z}+1$. Finally, applying union bound to all rounds $t$ that $t > \frac{4\log n}{z-1-\log z}+1$ completes the proof. 
\end{proof}

% \begin{lemma}
% \label{lemma:bound_variance}
%     For any $n\geq 2$ and $t\leq n$, with probability of at least $1-\frac{1}{n}$, the absolute values of the rewards in reward pool $\mathcal R_t$ are bounded by $2\sqrt{\sigma^2\log(n)}+1$. 
% \end{lemma}
\boundreward*
\begin{proof}
The maximum number of rewards in the agent's history is $n$. Each reward $X_i$ can be viewed as a sample from $X_i\sim\mathcal N(\mu_i,\sigma^2)$, where $\mu_i$ is the mean reward of the pulled arm. Based on Hoeffding's inequality, we have 
\begin{equation}
    \prob{X_i-\mu_i\geq \sqrt{2\sigma^2\log(n/\delta)}} \leq \delta/n.
\end{equation}
% that $X_i$ deviates from the mean $\mu_i$ by more than $\sqrt{2\sigma^2\log(n/\delta)}$ with probability of at most $\delta/n$.
Then with union bound, we have for all $(X_i)_{i=1}^n$, the probability of any of them being away from the mean by more than $\sqrt{2\sigma^2\log(n/\delta)}$ is smaller than $\delta$. As $\mu_i\in[0,1]$, we have for any $X_i, i\in[n]$, 
\begin{equation}
    -\sqrt{2\sigma^2\log(n/\delta)}\leq X_i\leq \sqrt{2\sigma^2\log(n/\delta)}+1
\end{equation}    
with probability of at least $\delta$. The mean of the rewards is in the same range. Finally, subtracting the mean, scaling the rewards by $\alpha$ and letting $\delta=1/n$ completes the proof.  
\end{proof}

\begin{lemma}
\label{lem:general chi-square upper bound} Let $Z = (Z_i)_{i = 1}^d$ be a vector of independent standard normal variables. Let $A \in \mathbb{R}^{d \times d}$ be an arbitrary matrix and $X = A Z$. Then for any $\varepsilon > 0$ and vector $v \in \mathbb{R}^d$, we have that
\begin{align*}
  \prob{\normw{X}{2}^2 \leq \varepsilon^2}
  \geq \prob{\normw{X + v}{2}^2 \leq \varepsilon^2}\,.
\end{align*}
\end{lemma}
\begin{proof}
Since $Z \sim \mathcal{N}(\mathbf{0}, I_d)$, we have that $X \sim \mathcal{N}(\mathbf{0}, A A\T)$. Let $Y = X + v$. Then we also have that $Y \sim \mathcal{N}(v, A A\T)$. The important properties of $X$ and $Y$ are that their covariance matrices are the same.

Now note the following. The quantity $\prob{\normw{Y}{2}^2 \leq \varepsilon^2}$ is the density of $Y$ within distance $\varepsilon$ of $\mathbf{0}$. Since $Y$ is centered at $\mu$, its density contours are symmetric and ellipsoidal, and the $\varepsilon$ constraint is a ball, we have that $\prob{\normw{Y}{2}^2 \leq \varepsilon^2}$ is maximized when $\mu = \mathbf{0}$. In other words, for any $v$, we have that
\begin{align*}
  \prob{\normw{X}{2}^2 \leq \varepsilon^2}
  \geq \prob{\normw{Y}{2}^2 \leq \varepsilon^2}\,.
\end{align*}
This concludes the proof.
\end{proof}

\end{document}